\documentclass{article}

\usepackage[main,final]{neurips_2026}

\usepackage[utf8]{inputenc} 
\usepackage[T1]{fontenc}    
\usepackage{hyperref}       
\usepackage{url}            
\usepackage{booktabs}       
\usepackage{amsfonts}       
\usepackage{nicefrac}       
\usepackage{microtype}      
\usepackage{xcolor}         

\usepackage{subcaption}    
\usepackage{lipsum}         
\usepackage{graphicx}
\usepackage{natbib}
\usepackage{doi}
\usepackage{wrapfig}
\usepackage{amsmath}
\usepackage{amssymb}
\usepackage{mathtools}
\usepackage{amsthm}
\RequirePackage{times}
\usepackage{cleveref}       

\RequirePackage{fancyhdr}
\RequirePackage{xcolor} 
\RequirePackage{algorithm}
\RequirePackage{algorithmic}
\RequirePackage{eso-pic} 
\RequirePackage{forloop}
\RequirePackage{url}
\RequirePackage{caption}

\theoremstyle{plain}
\newtheorem{theorem}{Theorem}[section]
\newtheorem{proposition}[theorem]{Proposition}
\newtheorem{lemma}[theorem]{Lemma}

\theoremstyle{definition}
\newtheorem{definition}[theorem]{Definition}

\theoremstyle{remark}
\newtheorem{remark}[theorem]{Remark}

\DeclareMathOperator*{\argmax}{arg\,max}

\newcommand{\reals}{\mathbb{R}}
\newcommand{\df}{\mathrm{df}}
\newcommand{\wass}{\mathrm{w}}
\newcommand{\states}{\mathcal{S}}
\newcommand{\actions}{\mathcal{A}}
\newcommand{\xc}{c}
\newcommand{\xC}{C}
\newcommand{\bs}{\mathcal{\xC}}
\newcommand{\distequiv}{\overset{\mathcal{D}}{=}}
\newcommand{\pibar}{\overline{\pi}}
\newcommand{\stationary}{\Pi}
\newcommand{\E}{\mathbb{E}}
\newcommand{\Uf}{U_f}
\usepackage{multirow}
\usepackage{enumitem}

\usepackage[textsize=tiny]{todonotes}

\title{Decoupling Time and Risk: Risk-Sensitive Reinforcement Learning with General Discounting}

\author{ 
	Mehrdad Moghimi \\
	Dept. of Mathematics and Statistics\\
	York University\\
	Toronto, Canada \\
	\texttt{moghimi@yorku.ca} \\
	\And
	Anthony Coache \\
	Dept. of Mathematics\\
	Imperial College London\\
	London, United Kingdom \\
	\texttt{a.coache@imperial.ac.uk} \\
	\And
	Hyejin Ku \\
	Dept. of Mathematics and Statistics\\
	York University\\
	Toronto, Canada \\
	\texttt{hku@yorku.ca} \\
}

\begin{document}

\maketitle

\begin{abstract}
	Distributional reinforcement learning (RL) is a powerful framework increasingly adopted in safety-critical domains for its ability to optimize risk-sensitive objectives. However, the role of the discount factor is often overlooked, as it is typically treated as a fixed parameter of the Markov decision process or tunable hyperparameter, with little consideration of its effect on the learned policy. In the literature, it is well-known that the discounting function plays a major role in characterizing time preferences of an agent, which an exponential discount factor cannot fully capture. Building on this insight, we propose a novel framework that supports flexible discounting of future rewards and optimization of risk measures in distributional RL. We provide a technical analysis of the optimality of our algorithms, show that our multi-horizon extension fixes issues raised with existing methodologies, and validate the robustness of our methods through extensive experiments. Our results highlight that discounting is a cornerstone in decision-making problems for capturing more expressive temporal and risk preferences profiles, with potential implications for real-world safety-critical applications.
\end{abstract}

\section{Introduction}
In the standard formulation of reinforcement learning (RL) problems, a key component is the discount factor, denoted by $\gamma$, which applies exponential discounting to future rewards. Discounting encodes time preferences of the agent, models unknown termination times, and generally reduces approximation errors and variance from distant rewards~\citep{petrik2008biasing,Amit.etal2020}, altogether leading to more stable learning algorithms. While mathematically convenient, this paradigm may create fundamental challenges in more complex settings. Choosing $\gamma\in(0,1)$ can render the agent's actions ineffective at altering return distributions over long timescales~\citep{Pires.etal2025}.
On the other hand, extending the agent's effective horizon via $\gamma\approx 1$ can be impractical as the variance of temporal-difference updates grows sharply, leading to severe training instability and slow convergence~\citep{Naik.etal2019a}.

Most importantly, exponential discounting \emph{cannot} capture some preference reversals that characterize humans' and animals' behaviors for various long-horizon planning tasks. Exponential discounting with $d_t=\gamma^t$ is the family for which relative time preference $d_{t+k} / d_{t} = \gamma^{k}$ is independent of when it is evaluated, whereas this ratio generally depends on $t$ for other discount functions (e.g., hyperbolic discounting with $d_t=(1+kt)^{-1}$). Under hyperbolic discounting, the one-step factor $d_{t+1}/d_t$ rises from $1/2$ at $t=0$ to $10/11$ at $t=9$ for $k=1$, so the same one-step delay is penalized heavily early on and barely at all later. This $t$-dependence produces preference reversals, which a stationary policy cannot capture. To illustrate, consider a hypothetical scenario from~\citet{Fedus.etal2019a}: suppose that someone offers to a risk-neutral agent, without any risk, either \$1 at time $t$ or \$1.1 at time $t+1$. Under exponential discounting, the agent would prefer the same reward for any time $t$. In contrast, for some parametrizations of hyperbolic discounting, the agent would select \$1 if $t=0$ but \$1.1 if $t=365$. This remains true for a risk-sensitive agent, as long as the reward distributions yield risk measures of the same monetary values.

\citet{green2004discounting} explained that preference inconsistencies, such as impulsivity or self-control, arise when choice options differ on many dimensions: (i) magnitude of the rewards, inherently embedded in the reward function; (ii) time elapsed before receiving the rewards; and (iii) likelihood of receiving the rewards. While some special cases of preference modeling admit consistent preferences across such dimensions, i.e., reward scaling for risk-neutral agents with an exponential discount factor, this is generally not the case and modeling time preferences through a single scalar discount factor can be overly restrictive.
To accurately describe time and risk preferences of any agent, an RL algorithm should support different forms of the discounting (second dimension) and objective function (third dimension), resulting in a wider range of preference profiles and optimal strategies.

Despite their central role in behavioral decision research, these dimensions have not been systematically integrated into RL models. On one hand, several works have considered general discounting schemes, including state-dependent~\citep{White2017,Pitis2019} and adaptive discount factors~\citep{Francois-Lavet.etal2016a,Xu.etal2018,Zahavy.etal2020}. On the other hand, a large body of research aimed to move beyond expectation-based goal functionals, leading to a plethora of risk-sensitive RL algorithms~\citep{Bauerle.Ott2011,Tamar.etal2017,DiCastro.etal2019a} including distributional RL~\citep{Morimura.etal2010a,Bellemare.etal2017a}. While these two lines of work address complementary aspects of preference modeling, they have largely been developed in isolation. We provide more discussion on related works in Appendix~\ref{sec:related-work}.


\begin{wrapfigure}{r}{0.5\textwidth}
\vspace{-20pt}
  \begin{center}
    \includegraphics[width=0.5\textwidth]{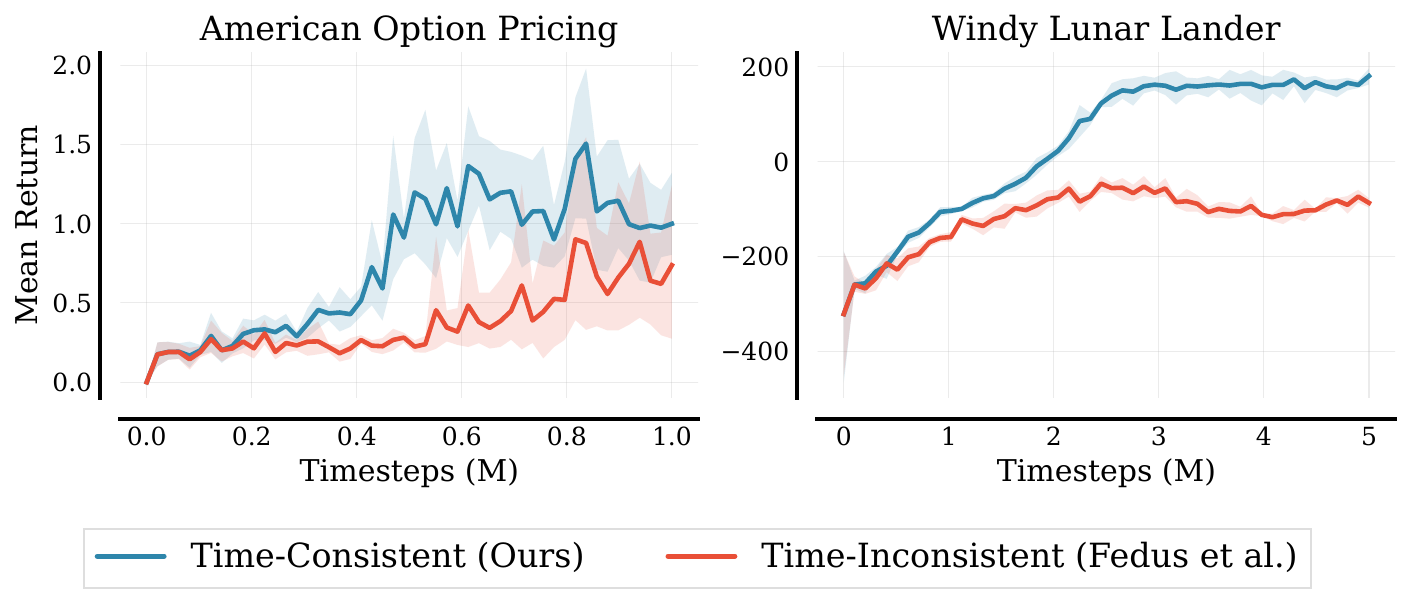}
  \end{center}
  \caption{\textbf{Impact of Time-Consistency.} Our time-consistent method (blue) outperforms the time-inconsistent approach of~\citet{Fedus.etal2019a} (red) by modeling the non-stationary policy.}
  \label{fig:tc}
\end{wrapfigure}

A notable work by~\citet{Fedus.etal2019a} attempted to introduce hyperbolic discounting into deep RL, but overlooked the fundamental need for time-dependent policies, as noted by~\citet{Schultheis.etal2022}.
Enforcing stationarity on a problem where the optimal strategy is inherently non-stationary violates \emph{time consistency}, which requires that a policy optimal at the initial time remains optimal at every future decision point. Figure~\ref{fig:tc} illustrates this critical distinction: our time-consistent approach, which properly accounts for non-stationarity, displays faster convergence to high returns and outperforms the time-inconsistent baseline. We further show in our experiments that these performance gains scale to high-dimensional Atari games, beyond simple environments.

This exposes a critical gap: how can we enable general discount functions \emph{and} risk-sensitive reasoning without incurring instability or abandoning the intuitive framework of optimizing a utility over total return? In this work, we answer by extending the theory of stock-augmented distributional RL to allow general discounting approaches \emph{and} a richer class of objective functions.

\textbf{Contributions.}
First, we introduce a unified stock-augmented distributional RL framework that accommodates general discount functions and optimized certainty equivalent (OCE) risk measures for finite-horizon problems (Section~\ref{sec:finite}). Building upon return distribution optimization (RDO)~\citep{Pires.etal2025}, we demonstrate that general discounting and risk-sensitive objectives can be seamlessly integrated using non-stationary policies. Through a goal-based wealth management problem, we illustrate how this framework accurately models complex preference profiles, capturing the preference reversal phenomenon often observed in human behavior (Section~\ref{sec:gbwm}).

Second, we resolve a key theoretical inconsistency in prior approaches for approximating general discounting with multi-horizon value functions~\citep{Fedus.etal2019a} through an explicit formulation of the required time-dependent optimal policy (Section~\ref{sec:multihorizon}). We empirically validate this time-consistent formulation on the high-dimensional Atari 2600 benchmark, demonstrating that resolving this theoretical inconsistency is not merely a conceptual exercise but directly translates to notable performance improvements in complex environments (Section~\ref{sec:atari}).

Finally, we formalize the theoretical nuances of extending static risk measures and general discounting to infinite-horizon settings (Section~\ref{sec:infinite}). By establishing the asymptotic risk-neutrality of smooth utility objectives under explicit regularity assumptions, we develop a principled approximation scheme. Specifically, we prove that an infinite-horizon risk-sensitive problem with general discounting can be approximated by a policy that follows these preferences up to a planning horizon, and subsequently switches to a stationary, risk-neutral policy with exponential discounting, providing a formal performance bound for this approximation.

\section{Motivating Example}
\label{sec:gridworld}

We first present a simple grid-world example to illustrate how risk preferences and temporal discounting affect optimal planning. We consider an $8\times8$ grid, where the agent starts at $(0,0)$ and only UP and RIGHT moves are available. Six terminal cells $T_\ell$ sit on the diagonal at $(\ell,\ell)$ for $\ell=1,\dots,6$. Since both moves advance along the diagonal, $T_\ell$ is reached after exactly $2\ell$ moves and its reward arrives at step $2\ell-1$ with weight $d_{2\ell-1}$. The agent cannot delay a chosen target, only switch to another, so the entire decision is \emph{which target to commit to}. The payoffs for $T_1,\dots,T_6$ are respectively $\{40,0\}$, $\{70,0\}$, $\{40,10\}$, a guaranteed $20$, $\{70,20\}$ and $\{40,30\}$, where each pair represents two equally likely outcomes. Crucially, the spread is not monotone in delay (i.e., $T_2$ is near but all-or-nothing, $T_4$ mid-delay and riskless, $T_6$ far and nearly riskless), so that impatience and risk aversion do not both favor near targets.

Entries in Table~\ref{tab:optimal_plans}, computed exactly, are the value of committing to each target under the agent's time-zero criterion, which is the objective the time-consistent agent maximizes, whereas the bold cells correspond to the best plan in each row. Three risk levels (expectation and entropic risk $\operatorname{Entr}_\beta(G) = -\beta^{-1}\log\E[e^{-\beta G}]$ with $\beta\in\{0.5, 1.5\}$) crossed with two discount functions (exponential discounting with $\gamma=0.8$ and hyperbolic discounting with $k=1$) give six distinct optimal plans.

\begin{table}[!ht]
\caption{Comparison of optimal plans under different criteria and discount functions.}
\label{tab:optimal_plans}
\begin{center}
\begin{small}
\begin{sc}
    \begin{tabular}{ll rrrrrrc}
    \toprule
    Criterion & Discount & $T_1$ & $T_2$ & $T_3$ & $T_4$ & $T_5$ & $T_6$ & Optimal \\
    \midrule
    $\mathbb{E}$
    & $\gamma=0.8$
    & $16.000$ & $\mathbf{17.920}$ & $8.192$ & $4.194$ & $6.040$ & $3.006$ & $T_2$ \\
    
    $\mathbb{E}$
    & $k=1$
    & $\mathbf{10.000}$ & $8.750$ & $4.167$ & $2.500$ & $4.500$ & $2.917$ & $T_1$ \\
    
    $\operatorname{Entr}_{0.5}$
    & $\gamma=0.8$
    & $1.386$ & $1.386$ & $\mathbf{4.648}$ & $4.194$ & $4.002$ & $2.961$ & $T_3$ \\
    
    $\operatorname{Entr}_{0.5}$
    & $k=1$
    & $1.386$ & $1.386$ & $2.895$ & $2.500$ & $\mathbf{3.229}$ & $2.874$ & $T_5$ \\
    
    $\operatorname{Entr}_{1.5}$
    & $\gamma=0.8$
    & $0.462$ & $0.462$ & $3.739$ & $\mathbf{4.194}$ & $3.146$ & $2.877$ & $T_4$ \\
    
    $\operatorname{Entr}_{1.5}$
    & $k=1$
    & $0.462$ & $0.462$ & $2.128$ & $2.500$ & $2.462$ & $\mathbf{2.794}$ & $T_6$ \\
    \bottomrule
    \end{tabular}
\end{sc}
\end{small}
\end{center}
\end{table}

The last column of Table~\ref{tab:optimal_plans} collects the six optimal plans, and all six differ. Holding the criterion fixed isolates the effect of discounting, which shows that
hyperbolic discounting pulls the target \emph{earlier} for the risk-neutral agent but \emph{later} for both risk-averse ones. The reason is that the two one-step discount curves cross, and hyperbolic discounting is more impatient early on and more patient later. This drives a risk-neutral agent to grab the near gamble and a risk-averse agent to wait for a safe, distant payoff. Comparing rows sharing a discount isolates risk aversion, and the optimum
traces a monotone risk-reward frontier in both discounting cases, with the agent selecting progressively less risky rewards.

\section{Stock-Augmented Return Distribution Optimization}\label{sec:stock-aug-DRL}
\textbf{Notation.}
Let $\Delta(\reals^n)$ denote the set of probability distributions over $\reals^n$, and $\df(\cdot)$ the distribution of a random variable. For $n\in\mathbb{N}$ outcomes, $\Delta(n)$ represents the $|n|$-dimensional probability simplex.

\textbf{Markov decision processes.}
The standard RL problem is modeled as a Markov decision process (MDP) tuple $(\states, \actions, P, R, \gamma)$. Here, $\states$ and $\actions$ are the state and action spaces, $P: \states \times \actions \to \Delta(\states)$ the transition function, $R: \states \times \actions \to \Delta(\reals)$ the reward distribution, and $\gamma \in (0, 1]$ the discount factor ($\gamma=0$ is excluded due to incompatibility with stock augmentation). The agent seeks a policy $\pi: \states \to \Delta(\actions)$ maximizing the expected discounted return, 
\(
G^\gamma_t \doteq \sum_{k=0}^{\infty} \gamma^k R_{t+k+1}.
\)
The superscript $\gamma$ emphasizes dependence on the discount factor. Since $R_t$ is a random variable resulting from taking action $A_t$ in state $S_t$, $G^\gamma$ implicitly depends on $\pi$.

\textbf{Return distribution optimization (RDO).}
\citet{Pires.etal2025} generalize this problem to optimizing arbitrary statistical functionals of multivariate returns, termed \emph{RDO}. Their framework uses a vector-valued pseudo-reward $R_{t} \in \bs \doteq \reals^m$ with a uniformly bounded first moment ($m\in\mathbb{N}$). To apply dynamic programming (DP), they introduce \emph{stock augmentation}, augmenting the state $s \in \states$ with a \emph{stock} vector $\xc \in \bs$. The stock tracks accumulated rewards recursively via $C^{\gamma}_{t+1} \doteq \gamma^{-1}(C^{\gamma}_{t} + R_{t+1})$, derived from $C^{\gamma}_{t} = \gamma^{-t}(C^{\gamma}_0 + \sum_{k=0}^{t-1} \gamma^k R_{k+1})$. The policy $\pi$ thus maps from the augmented state $(s, \xc) \in \states \times \bs$.

The objective is to find $\pi$ maximizing a functional $F_K$ over the total episode outcome, $C^{\gamma}_0 + G^{\gamma}_0$ (initial stock plus total discounted return), where $C^{\gamma}_0$ is fixed a priori or optimized over an interval. $F_K$ derives from a statistical functional $K: \Delta(\bs) \to \reals$ quantifying risk preferences over return distributions. Given a return distribution function $\eta^{\gamma}_\pi(s,c)\doteq \df(G^{\gamma}(s,c))$, the objective is $(F_K \eta^\gamma_\pi)(s, \xc) \doteq K\df(\xc + G(s, \xc))$ with $G(s, \xc) \sim \eta^\gamma_\pi(s, \xc)$. In particular, setting $K$ to the expectation recovers standard risk-neutral RL. The overall goal is then to find $\sup_{\pi} F_K \eta^\gamma_\pi$ at the initial state $(s_0, c_0)$.

A key DP challenge is that decisions occur at $t > 0$, while the objective spans from $t=0$. RDO resolves this via a crucial relationship between the total outcome and the time-$t$ state (derived in Appendix~\ref{app:anytime_proxy}):
\begin{equation}
	\label{eq:anytime_proxy} 
	C^{\gamma}_0 + G^{\gamma}_0 \distequiv \gamma^t \left( C^{\gamma}_t +  G^{\gamma}_t \right) .
\end{equation}

This scaling relationship reveals $C^{\gamma}_t + G^{\gamma}_t$ as an ``anytime proxy'' for the original objective. To leverage this proxy for optimization, they introduce the \textit{indifference to scaling} and \textit{indifference to mixtures} properties for the functional $K$~\citep[Definitions 9 \& 10 in][]{Pires.etal2025}, recalled in Appendix~\ref{app:finite}. The former ensures preferences are preserved under scaling by $\gamma$, allowing proxy improvements to reliably translate to the original objective over $C^{\gamma}_0 + G^{\gamma}_0$. The latter states that component-wise preference over outcome distributions implies preference over their probabilistic mixtures, guaranteeing the monotonicity of the Bellman operator.

The distributional DP methods in the RDO framework rely on the \emph{stock-augmented distributional Bellman operator} for a stationary policy $\pi$ and return distribution function $\eta$:
\begin{equation}\label{eq:operator-gamma}
	(\mathcal{T}^{\gamma}_\pi \eta)(s, \xc)  \doteq \df\left( R_{t+1} + \gamma G(S_{t+1}, C^\gamma_{t+1}) \right),
\end{equation}
where $(S_{t}, C_{t}) = (s, \xc)$, $(S_{t+1}, R_{t+1}) \sim P(\cdot | s, A_t)$, $A_t \sim \pi(s, \xc)$, and $G \sim \eta(S_{t+1}, (\xc + R_{t+1}) / \gamma)$. The return distribution $\eta^\gamma_\pi$ is the unique fixed point satisfying the distributional Bellman equation $\eta^\gamma_\pi = \mathcal{T}^{\gamma}_\pi \eta^\gamma_\pi$.

Moving from policy evaluation to control requires selecting actions that greedily improve $F_K$. A greedy policy $\pibar$ with respect to $\eta$ satisfies $F_K \mathcal{T}^{\gamma}_{\pibar} \eta = \sup_{\pi' \in \stationary} F_K \mathcal{T}^{\gamma}_{\pi'} \eta$, which is implemented by selecting the action $a^*$ at state $(s, \xc)$ such that
\[
a^*(s, \xc) \in \argmax_{a \in \actions} (F_K (\mathcal{T}^\gamma_{\pi_a} \eta))(s, \xc),
\]
where $\pi_a$ is the policy that deterministically selects action $a\in\actions$. 
For the important special case of expected utilities, where $F_K = \Uf$ for a function $f : \bs \to \reals$, this reduces to maximizing the expected utility of the immediate outcome $\E f(\xc + G(s, \xc, a))$ with $G(s, \xc, a) \sim (\mathcal{T}^\gamma_{\pi_a} \eta)(s, \xc)$. These operators form the basis of distributional value and policy iteration algorithms for finding optimal policies with respect to $F_K$.

\section{General Discounting in Finite Horizon}
\label{sec:finite}
We extend stock-augmented distributional RL to general time preferences. Let $d: \mathbb{N}_0 \to \mathbb{R}^+$ be a general discount function, and define the \emph{one-step discount factor} as $\hat{d}_t \doteq d_{t+1} / d_t$ where $d_t \doteq d(t)$. Intuitively, we assume $d_0=1$ and require $d_t$ to be non-increasing ($d_{t+1} \leq d_t$ for all $t \geq 0$). This aligns with the economic principle that future rewards are not valued more than present ones and accommodates exponential, hyperbolic, and quasi-hyperbolic models. Consequently, $\hat{d}_t \in (0, 1]$ for all $t$.

The discounted return from time $t$ over a finite horizon $T \in \mathbb{N}$ is $G^d_t \doteq \frac{1}{d_t}\sum_{k=t}^{T-1} d_k R_{k+1}$. We also define the stock $C^d_t = d_t^{-1}(C^d_0 + \sum_{k=0}^{t-1} d_k R_{k+1})$, which updates via $C^d_{t+1} = \hat{d}_t^{-1}(C^d_t + R_{t+1})$. It follows directly that setting $d_t=\gamma^t$ recovers the exponentially discounted stock in RDO, and the stock provides an ``anytime proxy'' for the total outcome, analogous to~\eqref{eq:anytime_proxy}. The total unscaled return from time zero, which is the target of a static utility function, can be expressed recursively using the current stock $C^d_t$ and future return $G^d_t$ (derived in Appendix~\ref{app:anytime_proxy}):
\begin{equation}
	\label{eq:anytime_proxy_d}
	C^d_0 + G^d_0 \distequiv d_t \left(C^d_t + G^d_t\right) .
\end{equation}

This recursive structure enables a DP operator. However, the non-stationarity of $\hat{d}_t$ renders the operator and optimal policy time-dependent. We therefore assume a finite and known horizon $T$ in this section, ensuring a well-posed problem solvable via backward induction. This assumption is necessary because, for general discount functions (e.g., hyperbolic), $\hat{d}_t \to 1$ as $t \to \infty$, violating the contraction mapping condition required for standard infinite-horizon DP. We address the infinite-horizon case in Section~\ref{sec:infinite}.

A non-stationary policy is a sequence $\pi = (\pi_0, \dots, \pi_{T-1})$, with $\pi_t: \states \times \mathcal{C} \to \Delta(\actions)$. For a return distribution function $\eta^\pi_t(s,c)\doteq \df(G^d_t(s,c))$, analogous to~\eqref{eq:operator-gamma}, we define the \emph{time-dependent distributional Bellman operator} for policy $\pi_t$ and next-step distribution $\eta_{t+1}$ as
\[
(\mathcal{T}^d_{\pi_t} \eta_{t+1})(s, c) \doteq \mathrm{df}\left( R_{t+1} + \hat{d}_tG^d_{t+1}(S_{t+1}, C^d_{t+1}) \right),
\]
where $S_{t+1} \sim P(\cdot | s, A_t)$, $R_{t+1} \sim R(\cdot | s, A_t)$, $A_t \sim \pi_t(\cdot | s, c)$, and $G^d_{t+1} \sim \eta^\pi_{t+1}(S_{t+1}, (c + R_{t+1}) / \hat{d}_t)$. Within this framework, agents use the learned $\eta^\pi$ to reconstruct the total outcome distribution via~\eqref{eq:anytime_proxy_d} and optimize a static functional over it. This necessitates new theoretical guarantees for stock- and time-augmented distributional DP.

Adopting non-stationary return distributions relaxes the indifference to scaling assumption by directly optimizing the initial return distribution:
\[
(F_K \eta^\pi_t)(s, \xc) \doteq K\df(d_t\xc + d_tG(s, \xc))/d_t.
\]
If the functional satisfies scaling indifference, this naturally recovers the exponential discounting objective in RDO. This flexibility enables our framework to accommodate optimized certainty equivalent (OCE) risk measures, a broad class of convex risk measures originally developed for financial risk management~\citep{Ben-Tal.Teboulle2007}. For a random variable $G$ and utility $f:\mathbb{R} \to [-\infty, \infty)$, the OCE is:
\begin{equation}\label{eq:oce}
	\operatorname{OCE}_f(G) \doteq \max_{c_0 \in \mathcal{C}} \{-c_0 + \mathbb{E}[f(c_0 + G)]\}.
\end{equation}
Optimizing this metric involves finding the optimal return distribution for any initial stock, then selecting the best $c_0$. Our DP framework solves the former simultaneously for all $c \in \mathcal{C}$ by setting the objective to $(U_f \eta_t)(s, c) \doteq \mathbb{E}[f(d_t c + d_t G_t)]/d_t$. The OCE value is then obtained by maximizing the scalar function $J(c_0) \doteq -c_0 + (U_f \eta^*_0)(s_0, c_0)$ over the learned value surface via grid search on a discretized $\bar{\mathcal{C}} \subset \mathcal{C}$. Appendix~\ref{app:risk_measures} illustrates $f$ for common risk measures; some satisfy scaling indifference (e.g., Mean-CVaR), while others (e.g., Entropic Risk, Mean-Variance) do not.

In finite-horizon settings, the introduction of the stock variable restores a recursive structure even under arbitrary discounting. Through the “anytime proxy” identity, the augmented state and time contain exactly the information needed to evaluate the original time-$0$ objective from any decision point onward. This makes DP possible with time-dependent Bellman operators and, in turn, leads to backward induction over non-stationary policies, capturing time preferences as in the hypothetical scenario from~\citet{Fedus.etal2019a}. Lemma~\ref{lem:time-monotonicity} (proof in Appendix~\ref{app:finite1}) formalizes this principle and is then used to establish the optimality of our approach in Theorem~\ref{thm:backward-induction-optimality} (proof in Appendix~\ref{app:finite2}).

\begin{lemma}[Time-Dependent Monotonicity]
	\label{lem:time-monotonicity}
	Let $K$ be an objective functional satisfying the indifference to mixture property. Then for any time $t < T$, any policy $\pi_t$ and return distribution functions $\eta_{t+1}$ and $\eta'_{t+1}$, $F_K \eta_{t+1} \geq F_K \eta'_{t+1}$ implies $F_K (\mathcal{T}^d_{\pi_t} \eta_{t+1}) \geq F_K (\mathcal{T}^d_{\pi_t} \eta'_{t+1})$.
\end{lemma}

\begin{theorem}[Optimality of Distributional Backward Induction]
	\label{thm:backward-induction-optimality}
	Let the problem horizon $T\in\mathbb{N}$ be finite and known, and let the objective functional $K$ be indifferent to mixtures. Assume that the sequence of return distributions $\eta^*=(\eta^*_0, \dots, \eta^*_T)$ is generated by the following backward induction algorithm:
	\begin{itemize}[noitemsep, topsep=0pt]
		\item \textbf{Base Case:} For all $(s,c) \in \states \times \mathcal{C}$, set $\eta^*_T(s,c) = \delta_0$.
		\item \textbf{Recursive Step:} For $t = T-1$ down to $0$, compute $\eta^*_t \doteq \mathcal{T}^d_* \eta^*_{t+1}$,
		where $\mathcal{T}^d_*$ is the \emph{greedy time-dependent distributional Bellman operator}, defined such that for any function $\eta_{t+1}$ and for all $(s,c)\in \states \times \mathcal{C}$, it returns the function $\eta_t$ satisfying:
		\[
		(F_K \eta_t)(s,c) = \sup_{\pi_t} (F_K (\mathcal{T}^d_{\pi_t} \eta_{t+1}))(s,c).
		\]
	\end{itemize}
	Then, the sequence $\eta^*$ is a sequence of optimal return distributions. Furthermore, the non-stationary policy $\pi^* = (\pi^*_0, \dots, \pi^*_{T-1})$, where each $\pi^*_t$ is a \emph{greedy policy} with respect to $\eta^*_{t+1}$, is optimal for any initial augmented state.    
\end{theorem}


\begin{remark}[On Uniqueness and Lipschitz Continuity]
	The greedy operator $\mathcal{T}^d_*$ may lack uniqueness, as multiple actions or stochastic mixtures can yield optimal distributions with respect to $F_K$~\citep[see e.g. Section 7.5 of][]{Bellemare.etal2023}. Thus, the optimal policy $\pi^*$ may not be unique. Furthermore, the exact inner problem's optimality guarantee relies only on monotonicity, which contrasts with the infinite-horizon setting in RDO, where Lipschitz continuity is required for convergence. However, the full OCE problem strictly requires the Lipschitz property to bound discretization errors in the outer $c_0$ optimization and control error propagation in approximate DP.
\end{remark}

While not required for exact planning, Lipschitz continuity becomes essential for analyzing the error in any practical algorithm that must approximate the return distribution. Proposition~\ref{prop:bound-approx-backward-induction} provides a bound for the total suboptimality of the approximate algorithm, accounting for both inner approximation and outer discretization errors (see proof in Appendix~\ref{app:finite3}).

\begin{proposition}[OCE Optimality Bound]
	\label{prop:bound-approx-backward-induction}
	Let $f: \mathbb{R} \to \mathbb{R}$ be an $L$-Lipschitz utility function. Let $\pi^*$ be the optimal policy from Theorem~\ref{thm:backward-induction-optimality}, and let $\hat{\pi}$ be the policy derived from an approximate backward induction with one-step distributional projection errors $(\varepsilon_0, \dots, \varepsilon_{T-1})$. Let the outer optimization be performed over a finite grid $\bar{\mathcal{C}}$ with spacing $\delta$. If $\hat{c} \in \bar{\mathcal{C}}$ denotes the initial stock selected by the algorithm, and $\mathcal{E}_{\text{DP}} \doteq L \sum_{t=0}^{T-1} d_t \varepsilon_t$ denotes the cumulative DP error term, the total suboptimality is bounded by
	\[
	\operatorname{OCE}_f(\eta^*_0) - \left( -\hat{c} + \mathbb{E}[f(\hat{c} + G^{\hat{\pi}})] \right) \leq 2 \mathcal{E}_{\text{DP}} + \frac{(1 + L)\delta}{2}.
	\]
\end{proposition}
\section{Multi-Horizon Approximation for General Discounting}
\label{sec:multihorizon}
Many discount functions, such as hyperbolic, can be expressed as a weighted integral of exponential discount functions,
\begin{equation}\label{eq:discount-function}
	d_t = \int_0^1 \gamma^t w(\gamma) d\gamma,
\end{equation}
where $w:[0,1]\rightarrow\mathbb{R}$ is a weighting function over the discount factor $\gamma$. This integral can be approximated with a discrete Riemann sum over a fixed set of $m$ exponential discount factors $\Gamma = \{\gamma_1, \dots, \gamma_m\}$ as $d_t \approx \tilde{d}_{t} \doteq \sum_{i=1}^m w_i \gamma_i^t$, where the weights $w_i = w(\gamma_i)(\gamma_i-\gamma_{i-1})$ are derived from the weighting function.
\citet{Fedus.etal2019a} uses this fact to derive a stationary value function $Q^d(s,a) = \int_0^1 Q^\gamma(s,a) w(\gamma) d\gamma$ under the general discount functions of the form~\eqref{eq:discount-function}. 

While this technique can be practically effective, it is theoretically incoherent. The resulting stationary policy, $\pi(s)$, greedily optimizes a local objective that resets at every time step, and is therefore not guaranteed to be optimal for any single, fixed global objective. As~\citet{Schultheis.etal2022} note, this stationary policy is also structurally incapable of modeling the preference reversal in hyperbolic discounting. To correctly optimize a global objective with any general discount function $d_t$, the passage of time must be an explicit part of the state representation. As formalized by~\citet{Schultheis.etal2022} for continuous-time problems, the value function must be non-stationary, which in discrete time corresponds to a time-dependent Bellman optimality equation:
\begin{equation*}
	Q_t^d(s, a) =  R(s,a) + \hat{d}_t \max_{a^\prime}Q_{t+1}^d(s^\prime, a^\prime).
\end{equation*}
This formulation is theoretically coherent, as it finds the true optimal non-stationary policy that maximizes the single, fixed objective function defined at the start of the episode $\mathbb{E}[\sum_{t=0}^T d_t R_{t+1}]$. It provides a principled ``top-down'' solution for finding an optimal plan under the assumption of pre-commitment.

Compared to the backward induction algorithm in Section~\ref{sec:finite}, the multi-horizon approach introduces an approximation error; however, simultaneously predicting returns over multiple time horizons serves as a strong auxiliary task that promotes richer state representations~\citep{Jaderberg.etal2016a}. Using the recursive structure of the stock-augmented distributional RL framework, we express the return in terms of $m$ exponential discounts. Accordingly, the discounted return for $\gamma_i$ is $G^{\gamma_i}_t = \sum_{k=0}^{T-t-1} \gamma_i^k R_{t+k+1}$. The unscaled outcome under the approximate discount function $\tilde{d}_t$ can then be written recursively, analogous to~\eqref{eq:anytime_proxy_d} (derived in Appendix~\ref{app:anytime_proxy}):
\begin{equation}
	\label{eq:anytime_proxy_m}
	C^{\tilde{d}}_0 + G^{\tilde{d}}_0 \distequiv \tilde{d}_{t} \left(C^{\tilde{d}}_t + \sum_{i=1}^m w_{i,t} G^{\gamma_i}_t\right),
\end{equation}
where the time-dependent weights are $w_{i,t} \doteq w_i \gamma_i^t/ \tilde{d}_{t}$. Because these weights vary with time, the resulting decision problem is non-stationary. Nevertheless, unlike general discount functions for which $\hat{d}_t$ may approach 1, the multi-horizon approximation remains well-defined in the infinite-horizon setting as the effective discount factor $\hat{\tilde{d}}_t$ converges to $\gamma_m < 1$ as $t \to \infty$.

This formulation yields a practical algorithmic framework in which the agent learns a set of distributional value functions $\{\eta^{\gamma_i}\}_{i=1}^m$, each corresponding to an exponentially discounted problem. Notably, all predictions are conditioned on a single aggregate stock $C^{\tilde{d}}_t$, rather than maintaining separate stocks for each $\gamma_i$, since the policy is fully determined by the augmented state $(s_t, C^{\tilde{d}}_t)$. At decision time $t$, given state $(s_t, C^{\tilde{d}}_t)$, the agent uses the learned value distributions to reconstruct the distribution of the total outcome via~\eqref{eq:anytime_proxy_m}. Then, for each action, it aggregates the $m$ exponential return distributions quantile-wise using the current weights $w_{i,t}$ and selects the action that maximizes the objective functional $K$.

\section{Extension to Infinite Horizon}
\label{sec:infinite}

Extending the theoretical guarantees of our backward induction algorithm to the infinite-horizon setting requires additional structure. We therefore study a principled approximation scheme that yields performance guarantees under further assumptions. Specifically, this analysis requires understanding the asymptotic behavior of both the discount and objective functions. Throughout this section, we further assume that $d_t$ converges to zero and is summable, i.e., $\lim_{t \to \infty} d_t = 0$, and $\sum_{t=0}^{\infty} d_t < \infty$ (e.g., requiring $b > 1$ for generalized hyperbolic discounting of the form $(1+kt)^{-b}$).

\subsection{Discount Function's Asymptotic Behavior}
Applying general discounting to infinite-horizon problems poses a primary challenge: ensuring the model remains well-posed. A key failure mode occurs if the one-step discount factor $\hat{d}_{t}$ approaches $1$ as $t \to \infty$, as is the case with hyperbolic discounting. In this scenario, the Bellman operator loses its contraction property. To resolve this, we require a discount function where the one-step factor asymptotically converges to a limit strictly less than one.  We present one such construction method here, with an alternative approach detailed in Appendix~\ref{app:modified_integral_representation}.

A straightforward approach is to enforce a strict upper bound $\gamma_{\text{tail}} < 1$. By capping the original factor, we define a modified one-step discount $\hat{d}_t^{\text{new}} \doteq \min(\hat{d}_{t}, \gamma_{\text{tail}})$, which yields the corresponding full discount function $d_t^{\text{new}} = \prod_{k=0}^{t-1} \hat{d}_k^{\text{new}}$. Applied to a hyperbolic function, this construction maintains initial hyperbolic behavior before transitioning to standard exponential decay. Ultimately, this approach is equivalent to selecting a horizon $T$ and bounding the one-step discount function by $\hat{d}_{T}$, a strategy we adopt for our infinite-horizon analysis.

\subsection{Objective Function's Asymptotic Behavior}

Time-consistent static risk measures and general discount functions are inherently grounded in preferences established at the initial time step. When applied over long horizons, understanding their asymptotic behavior is critical because all future decisions are evaluated against these time-zero preferences. Depending on the utility function, these measures exhibit distinct asymptotic behaviors. For instance, pure CVaR can eventually become indifferent to any return distribution once sufficient stock is accumulated. In contrast, for OCE objectives with strictly increasing utilities, the optimal policy converges to risk-neutrality. The following lemma establishes this convergence for smooth utilities.

\begin{lemma}[Asymptotic Risk-Neutrality]
	\label{lem:asymptotic-neutrality}
	Let $f$ be a strictly increasing, twice continuously differentiable utility function with a bounded second derivative and $f'>0$ on the relevant compact outcome interval, and assume rewards are uniformly bounded. Then, as $t \to \infty$, the risk-sensitive action values converge uniformly to the risk-neutral action values. Consequently, the smooth utility objective becomes asymptotically risk-neutral in its action ordering. 
\end{lemma}

\begin{wrapfigure}{r}{0.5\textwidth}
  \begin{minipage}{0.5\textwidth}
  \vspace{-20pt}
   \begin{algorithm}[H]
	\caption{Distributional VI for Infinite-Horizon MDPs with General Discounting}
	\label{alg:infinite-horizon-approx}
	\begin{algorithmic}
		\STATE {\bfseries Input:} Planning horizon $T' < \infty$, functional $K$
		\STATE Set the discount factor $\gamma_{\text{tail}} \leftarrow d_{T'+1}/d_{T'}$
		\STATE Compute the optimal risk-neutral policy $\pi^*_{\text{RN}}$ and its return distribution $\eta^{\pi^*_{\text{RN}}}$ with $\gamma_{\text{tail}}$
		\STATE Initialize terminal value: $\hat{\eta}^*_{T'}(\cdot) \leftarrow \eta^{\pi^*_{\text{RN}}}(\cdot)$
		\FOR{$t=T'-1$ {\bfseries to} $0$}
		\STATE Compute $\hat{\eta}^*_t \leftarrow \mathcal{T}^d_* \hat{\eta}^*_{t+1}$ (Thm.~\ref{thm:backward-induction-optimality})
		\STATE Let $\hat{\pi}^*_t$ be a greedy policy with respect to $\hat{\eta}^*_{t+1}$
		\ENDFOR
		\STATE Construct the final policy $\hat{\pi}^*$ such that
		\[
		\hat{\pi}^*_t = 
		\begin{cases} 
			\hat{\pi}^*_t & \text{if } t < T' \\
			\pi^*_{\text{RN}} & \text{if } t \geq T'
		\end{cases}
		\]
		\STATE {\bfseries Output:} An approximately optimal policy $\hat{\pi}^*$
	\end{algorithmic}
\end{algorithm}
  \end{minipage}
\end{wrapfigure}

Motivated by this asymptotic risk-neutrality, we propose an approximate algorithm for the infinite-horizon case, inspired by the work of~\citet{Hau.etal2023b}. Because the agent's risk preferences naturally converge to risk-neutrality, we can consider a planning horizon $T'$ that uses the true risk-sensitive policy with general discounting, while approximating the problem beyond this point with a stationary, risk-neutral, exponentially discounted policy. We outline this procedure in Algorithm~\ref{alg:infinite-horizon-approx}. Importantly, this algorithm is not intended for practical implementation; rather, it serves as an analytical tool to study the theoretical soundness of infinite-horizon approximation.

With this approximation scheme defined, Theorem~\ref{thm:performance_bound} quantifies the theoretical performance loss incurred by switching to a risk-neutral policy at the truncation horizon $T'$. Rather than prescribing how to manually tune $T'$, it characterizes how preferences defined by $\beta$ and $d_t$ behave asymptotically and demonstrates that an agent following its true risk-sensitive preferences behaves nearly identically to one that switches to the risk-neutral policy at time $T'$, with an approximation gap that decays quadratically with the tail mass $D_{T'} \doteq \sum_{k=T'}^\infty d_k$. We establish this explicit convergence rate for the entropic certainty-equivalent objective. This inner objective bound can be combined with a grid search over the initial stock $c$ to analyze the full $\operatorname{OCE}_f$ objective (see Appendix~\ref{app:performance_bound}).

\begin{theorem}[Performance Bound for Approximate Infinite-Horizon Policy]
	\label{thm:performance_bound}
	Let the objective functional $K$ be the \emph{Entropic Risk Measure} with risk aversion coefficient $\beta > 0$ (corresponding to the exponential utility $f(x) = \beta^{-1}(1-e^{-\beta x})$). Let the discount function $d_{t}$ follow the truncation scheme with parameter $\gamma_{\text{tail}} < 1$ for $t \ge T'$, and define $D_t \doteq \sum_{k=t}^\infty d_k$. Let $\pi^*$ be the optimal non-stationary policy for the infinite-horizon problem and $\hat{\pi}^*$ the policy returned by Algorithm~\ref{alg:infinite-horizon-approx} with planning horizon $T'$. Then, writing $V^{\pi,\mathrm{CE}}_0$ for the entropic certainty-equivalent value of policy $\pi$ at time $t=0$, the suboptimality of the approximate policy is bounded by
	\[
	\| V^{*,\mathrm{CE}}_0 - V^{\hat{\pi}^*,\mathrm{CE}}_0 \|_\infty \le \frac{\beta \Delta_R^2 D_{T'}^2}{8} .
	\]
\end{theorem}

As shown, the approximation gap scales proportionally with the agent's risk aversion ($\beta$), the squared reward range ($\Delta_R^2$), and the squared tail mass of the discount function ($D_{T'}^2$). This indicates that for highly volatile environments or highly risk-averse agents, the theoretical difference between the two regimes takes longer to vanish. Ultimately, this bound serves as a theoretical validation of Algorithm~\ref{alg:infinite-horizon-approx}, clarifying the exact conditions on the discount function and risk measure under which switching to a stationary, exponentially discounted policy is mathematically sound.

\section{Experiment Results}
\label{sec:experiments}
In this section, we design a series of experiments to evaluate the performance of our framework. Further details regarding the environments, hyperparameters, and expanded results are provided in Appendices~\ref{app:rigor-algorithm} and~\ref{app:additional-results}. The implementation code is available in the supplementary material.

\subsection{Goal-Based Wealth Management}\label{sec:gbwm}
Inspired by the hypothetical scenario from~\citet{Fedus.etal2019a} and the investment problem proposed by~\citet{Schultheis.etal2022}, we devise a goal-based wealth management (GBWM) application, in which the agent must invest his or her wealth and determine when fulfilling specified goals over a long-term trading horizon.~\citet{das2022dynamic} developed a DP framework to explicitly solve a class of GBWM problems, while~\citet{das2020dynamic} considered Q-learning for the one-goal GBWM problem and~\citet{Dixon.Halperin2020a} used a probabilistic extension of Q-learning. This set of experiments provides an illustration of the preference profiles that the framework can capture.

Consider an agent who can invest funds to achieve some financial goals in the future. Each goal $\psi$ is characterized by a cost $c_t(\psi)$ and a utility $u_t(\psi)$ representing the agent's satisfaction from realizing the goal. The agent begins with an initial wealth $y_0\in\reals_+$ and, at each period $t\in\{0,\dots,T\}$, makes two decisions: (i) which financial goal $\psi_t\in\{0,\dots,\Psi_t\}$ to fulfill, and (ii) which investment portfolio $\ell_t\in\{1,\dots,L\}$ to follow over the next period using the remaining wealth?\footnote{For notational convenience, we assume that the trading horizon is discretized into a sequence of time steps with equal duration, and the 15 candidate portfolios are ordered by risk aversion from conservative to aggressive as in~\citet{das2022dynamic}. At every period, the agent can fulfil at most one goal which includes a ``no-goal'' option with $c_t(0) = u_t(0) = 0$.} At each period, the state is given by $s_t = (t, y_t)$, while the 2-tuple action $a_t = (\psi_t,\ell_t)$ corresponds to the goals and portfolio strategy. State transitions occur in two stages, where we first deduct the cost of the selected goal $y_{t+} = y_t - c_t(\psi_t)$ and then $y_{t+1}$ is obtained by investing $y_{t+}$ in the portfolio $\ell_t$. The reward is given by the goal's utility $r(s_t,a_t) = u_t(\psi_t)$. At the terminal time, one may include a final reward for the utility associated with remaining wealth $u_T(y_T)$, but it must be carefully designed to avoid hyper-sensitive policies.

\begin{figure}[!ht]
\centering
\includegraphics[width=0.65\linewidth]{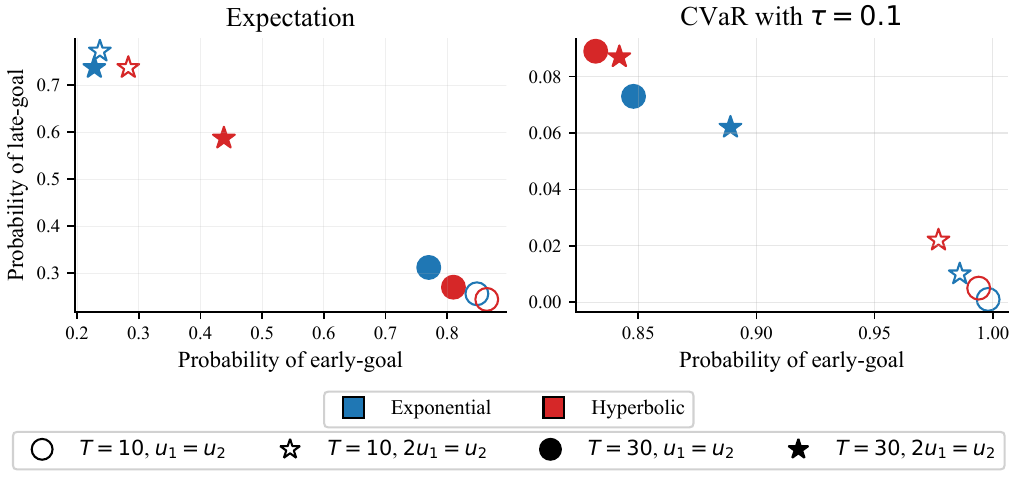}
\caption{\textbf{Preference reversals in GBWM.} Monte-Carlo probabilities of achieving the two goals for a risk-neutral (left) and risk-sensitive (right) agent, with various discounting (colors) and environment settings (shapes). Standard errors, not included for clarity, are provided in Appendix~\ref{app:gbwm}.}
\label{fig:gbwm}
\end{figure}


\begin{figure*}[!t]
\centering
\includegraphics[width=\linewidth]{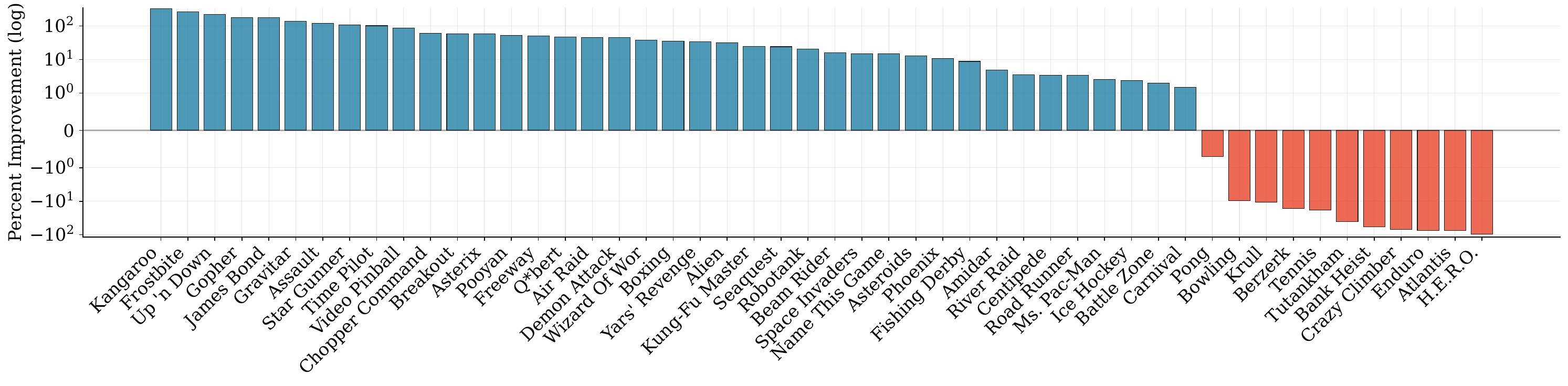}
\caption{\textbf{Relative performance of our Time-Consistent algorithm on 50 Atari games.} Bars show percent improvement over the time-inconsistent baseline (3 seeds). Our method outperforms in 39/50 games, with 39.89\% mean and 18.14\% median gains, highlighting the benefit of time-consistency under hyperbolic discounting.
} 
\label{fig:improvement}
\end{figure*}

In Figure~\ref{fig:gbwm}, we show Monte-Carlo estimations of the probabilities of achieving the early-goal with cost $c_{T/2}(1)=100$ and late-goal with cost $c_{T}(1)=150$ under the policy learned by Algorithm~\ref{alg:RIGOR_fixed_condensed2}. Estimates are computed over 10,000 episodes and averaged across three different runs. We ran our algorithm with both exponential and hyperbolic discounting and observed that when the two goals have the same utility ($u_{T/2}(1) = u_{T}(1) = 1000$, circles), the agent accomplishes a goal if he or she has the wealth available, as there is no incentive in waiting. There is little to no difference for risk-neutral agents between a 10-period (hollow) or 30-period (filled) problem. We then compare these observations with the case where the late-goal is twice as valuable as the early-goal ($u_{T/2}(1) = 1000$, $u_{T}(1) = 2000$, stars). Then, the hyperbolic discounting displays substantially larger probabilities of achieving the first goal, visible as a shift of the red filled star relative to its exponential counterpart, reflecting the impatience and preference reversals of the agent. This shows that an exponentially discounted optimal policy fails to capture preference reversals, in a similar manner to the hypothetical scenario described in the introduction from~\citet{Fedus.etal2019a}.


For risk-sensitive agents, the chosen portfolios tend to be more conservative and typically support the purchase of a single goal, leading to smaller expected total utilities under the policy learned by Algorithm~\ref{alg:RIGOR_fixed_condensed2}. In fact, the preference reversals are less pronounced compared to risk-neutral agents, where the strong risk-aversion drives the agent to secure a goal as soon as there is an opportunity. The effect is not reversed but saturated, as observed by the large probabilities of achieving the early-goal in the right panel of Figure~\ref{fig:gbwm} (see also Table~\ref{tab:gbwm2} for probabilities and their standard deviations).

\subsection{Atari 2600}\label{sec:atari}
To demonstrate scalability, we evaluate our time-consistent multi-horizon approximation (Algorithm~\ref{alg:RIGOR_fixed_condensed}) against its time-inconsistent variant on Atari 2600~\citep{bellemare13arcade}. Our goal is to verify that resolving this inconsistency improves performance over this variant under the same hyperbolic discount, rather than claiming state-of-the-art results. Comparisons with exponentially discounted baselines are in Appendix~\ref{app:atari}. Agents were trained for 40 million frames. Although shorter than the standard 200 million, this budget sufficed to reveal clear differences in policy quality.

The results, summarized in Figure~\ref{fig:improvement}, demonstrate a clear advantage for the time-consistent approach across 50 Atari games. Averaged over three seeds per game, the time-consistent agent achieves mean and median improvements of 39.89\% and 18.14\%, respectively, over the time-inconsistent baseline. This gap indicates that the incoherence of enforcing stationarity under hyperbolic objectives is a practical limitation hindering learning. To confirm these gains stem from time-consistency rather than including time in the state, we compare against other agents, including a Time-Aware Time-Inconsistent agent in Appendix~\ref{app:atari}. Our method retains a considerable advantage given identical state information, confirming the improvement arises from the time-consistent objective.

\section{Conclusion}
\label{sec:conclusion}
In this paper, we introduced a unified stock-augmented distributional RL framework supporting diverse time and risk preferences. Our approach accommodates general discount functions and OCE risk measures, outperforms prior precommitment algorithms with general discounting~\citep{Fedus.etal2019a}, and captures human-like preference reversals. While exponential discounting remains adequate when long-horizon problems require strong regularization for stability, our results highlight the value of general discounting for modeling preference inconsistencies, economically grounded term structures (Appendix~\ref{sec:discount-CIR}), or human-like temporal behaviors.

There are several promising future directions.  First, agents could adaptively learn their own time preferences by optimizing multi-horizon weights, similar to meta-gradient RL~\citep{Xu.etal2018, Zahavy.etal2020}. Furthermore, the multi-horizon framework also enables optimizing multi-dimensional utility functions capturing trade-offs across timescales, operating on the joint distribution of the multi-horizon return vector and thus requiring learning the full joint return distribution~\citep[see e.g.,][]{Zhang.etal2021d, Wiltzer.etal2024a}. Finally, while we lay the theoretical foundation in the value-based setting for discrete actions, our methods are currently limited to this regime; extending this framework to policy-gradient or actor-critic algorithms is an important direction, particularly for continuous control.

\begin{ack}
We thank the Digital Research Alliance of Canada (\href{https://alliancecan.ca/en}{alliancecan.ca}) for providing the computational resources used in this work. This research was supported in part by the Natural Sciences and Engineering Research Council of Canada (NSERC) under Discovery Grant No.~504486. The second author also gratefully acknowledges support from NSERC (grant PDF-2024-587484).
\end{ack}

\bibliography{MyLibrary} 
\bibliographystyle{icml2026}

\newpage
\appendix

\section{Related Work}\label{sec:related-work}

\paragraph{General Discounting in RL.}
Many RL frameworks primarily apply an exponential discount factor as a regularization tool to ensure convergence and stabilize training. This choice, however, imposes structural constraints on how future rewards are discounted, limiting the modeling expressiveness of the resulting algorithms. Furthermore, alternative discounting schemes naturally appear in stochastic environments and model-based settings. For instance, we show in Appendix~\ref{sec:discount-CIR} that, under some model assumptions, the risk-neutral expectation of the price of a zero-coupon bond leads to a deterministic hyperbolic discount function. These considerations motivated the development of RL frameworks that admit more general discounting mechanisms beyond the exponential form.

One major approach redefines the discount function to be explicitly time-dependent. \citet{Lattimore.Hutter2014} provide a complete characterization of time consistency within this model, proving that a policy remains consistent if and only if the agent's relative preferences for future rewards remain constant over its lifetime. Another line of research treats the discount factor as an adaptive parameter, to be tuned via meta-learning~\citep{Xu.etal2018, Zahavy.etal2020} or adjusted dynamically to improve training stability~\citep{Francois-Lavet.etal2016a}. A third, more foundational approach provides a formal basis for agents to have state- and action-dependent discount rates~\citep{White2017,Pitis2019}.

Within this landscape, several distinct paradigms for handling general discounting have emerged. \citet{Fedus.etal2019a} proposed a practical, bottom-up approach to approximate hyperbolic discounting by aggregating the value functions of multiple agents, each with a different exponential discount factor. In contrast,~\citet{Schultheis.etal2022} developed a theoretically principled, top-down model that embraces the time-dependency of the problem, leading to a non-stationary optimal policy. Taking a different path,~\citet{Bayraktar.etal2024} re-framed the issue as a game-theoretic problem between an agent's present and future selves, solving for a stable (but not necessarily $t_0$-optimal) equilibrium policy. Our work builds most directly on the principled, time-dependent optimal control paradigm.

To be comprehensive, we mention that the average-reward formulation, which optimizes the long-run reward per time-step, offers another possible approach for long-horizon problems. We do not take this approach, however, because it creates a mismatch when the goal is to optimize a utility function (under the expected theory~\citep{von2007theory}) over total accumulated return. A static utility function $f(G)$, meant to evaluate the risk profile of the total return distribution, is conceptually incompatible with an objective based on per-step averages.

\paragraph{Risk-Sensitive RL.}
Properly accounting for risk preferences of the agent is as crucial as accounting for time preferences. In many real-life applications such as resource management, financial investment or healthcare decision-making, ignoring environmental uncertainty can cause catastrophic outcomes. A large body of work in risk-sensitive RL has sought to move beyond expectation-based goal functionals. Many contributions have focused on classes of static risk measures~\citep[see e.g.,][]{Bauerle.Ott2011,Tamar.etal2012a,DiCastro.etal2019a}, while others have investigated robust Markov decision processes for optimizing worst-case expectations~\citep[see e.g.,][]{Osogami2012}. 

Parallel to these developments, distributional RL~\citep{Morimura.etal2010a,Bellemare.etal2017a} provided a framework for modeling the full (discounted) return distribution, allowing agents to capture parametric uncertainty and optimize risk-sensitive objectives with more stable learning algorithms. Our work falls in the realm of these works. Alternatively, a class of dynamic risk measures has been developed to ensure time-consistency by construction~\citep[see e.g.][]{Ruszczynski2010}, for which many authors have constructed time-consistent RL algorithms~\citep[see e.g.][]{Tamar.etal2017, Coache.Jaimungal2023}. These alternative works, however, fall outside the scope of our work with their limited interpretability.

\paragraph{Time-Consistency and State Augmentation.}
An agent is considered time-consistent if the optimal plan it formulates at the beginning of an episode remains optimal at all future decision points. However, when an agent's preferences are not based on simple expected returns with exponential discounting, this property is often lost, leading to policies that are dynamically unstable. The literature has converged on state augmentation as the primary mechanism to restore tractability and consistency, and this augmentation serves two distinct but related purposes.

The first role of state augmentation is to make the problem amenable to standard DP techniques, especially when the optimal policy depends on the entire history of past rewards, making the search space of policies intractable.~\citet{Bauerle.Ott2011} first introduced the concept of a stock to solve the CVaR optimization problem, demonstrating that the search for an optimal policy could be reduced from the space of history-dependent policies to the much smaller space of Markov policies defined on an augmented state space that tracks accumulated past costs. This powerful idea was later extended to the optimization of more general expected utility functions~\citep{Bauerle.Rieder2014} and spectral risk measures~\citep{Bauerle.Glauner2021}, and related notions such as ``time-awareness'' state augmentation have been shown to improve policy stability in RL~\citep{Pardo.etal2018}. Recently,~\citet{muni2026reward} addressed the sparse-reward issue in the augmented-state formulation of~\citet{Bauerle.Ott2011} by redistributing the signal across time. The second, more conceptual role of state augmentation is to serve as an anytime proxy for the initial return distribution~\citep{Moghimi.Ku2025a, moghimi2025risk, Pires.etal2025}. Static utility and risk measures are, by definition, functionals of the entire distribution of returns as seen from the initial time step. To make a decision at any future time that remains consistent with the initial objective, the agent requires information about how its current actions will affect that original distribution, which is exactly what the stock provides.

\section{Discounting in Financial Applications}
\label{sec:discount-CIR}

In many RL applications, the discount factor is primarily a regularization tool, introduced to ensure convergence and stabilize training, and its value is often tuned as a hyperparameter. However, in domains like finance, the discount factor has a precise economic meaning: it represents the present value of future risk-free cash flows, which is governed by the term structure of interest rates. While prior work, such as~\citet{Fedus.etal2019a}, has explored various discount functions derived from hazard priors, the direct connection of these functions to financial models is unclear.

In finance, the short-term interest rate is often modeled as a stochastic process. To derive a discount function, we consider the price of a zero-coupon bond at time $t$, denoted by $P(t,T)$, which pays \$1 at a future maturity date $T>t$. This price represents the value of a guaranteed dollar at time $T$, and is related to the path of the instantaneous short-rate, denoted by $(r_u)_{u\in(t,T)}$. Under no-arbitrage conditions, this price is given by the risk-neutral expectation of the discounted payoff:
\[
P(t,T) = \mathbb{E} \left[ e^{-\int_t^T r_u du} \right].
\]
Our deterministic discount function from the perspective of time zero is precisely the price of a zero-coupon bond that matures at time $t$, i.e., $d_{t} \equiv P(0,t)$. If the interest rate process $r_u$ is guaranteed to be non-negative, then the integral is non-negative and thus $d_{t} < 1$ for $t>0$, in line with our non-increasing assumption on the discount factor.

Seminal models from~\citet{vasicek1977equilibrium} or Cox-Ingersoll-Ross (CIR)~\citep{cox1985theory} provide analytic solutions for this bond price. The CIR model describes the evolution of $r_t$ as 
\[
dr_t = a(b - r_t)dt + \sigma \sqrt{r_t} dW_t,
\]
where $a$ is the speed of mean reversion, $b$ is the long-term mean rate, $\sigma$ is the volatility, and $dW_t$ is a Wiener process. The CIR model guarantees non-negative interest rates as long as $2ab \geq \sigma^2$. In this case, the bond price has an affine term structure solution of the form $P(t,T) = A(t,T)e^{-B(t,T)r_t}$, where the functions $A(t,T)$ and $B(t,T)$ depend only on the time to maturity $\tau \doteq T-t$ and $h \doteq \sqrt{a^2 + 2\sigma^2}$:
\begin{equation*}
	A(t,T) = \left[ \frac{2h e^{(a+h)\tau/2}}{ (a+h)(e^{h \tau}-1) + 2h } \right]^{2ab/\sigma^2} \quad \mbox{and} \quad B(t,T) = \frac{2(e^{h \tau}-1)}{(a+h)(e^{h \tau}-1) + 2h}.
\end{equation*}

The general nature of $d_{t}$ in our model provides the flexibility to choose a discount function that is well-motivated by the specific application domain. To obtain our discount function in the CIR model, we set $t=0$ and $T=t$ (and thus $\tau=t$), which yields $d_{t} = A(0,t)e^{-B(0,t)r_0}$. We provide a comparison of various discounting approaches in Figure~\ref{fig:discounts}. Besides this CIR model, other methods can be used for modeling time preference, such as deriving the discount function directly from a hazard rate model. Our framework is sufficiently general to accommodate such choices, providing a principled foundation for risk-sensitive decision-making under sophisticated and realistic models of time.

\begin{figure}[!ht]
	\centerline{\includegraphics[width=0.6\linewidth]{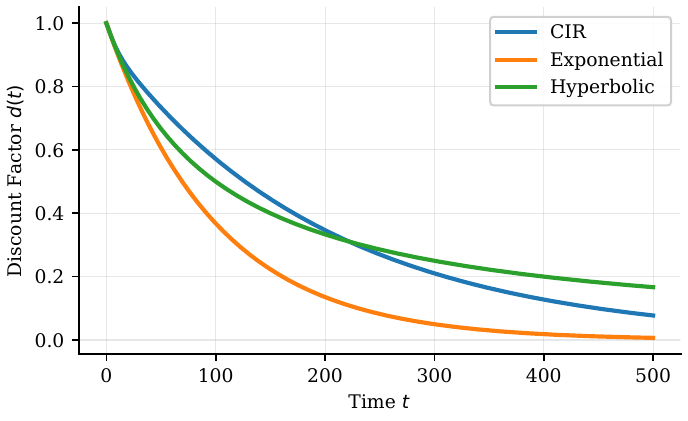}}
	\caption{\textbf{From interest rates to discount functions.} Comparison of exponential with $\gamma = 0.99$ (Orange), hyperbolic with $k = 0.01$ (Green), and CIR-based discount factors with parameters $a = 0.05$, $b = 0.01$, $\sigma = 0.1$, and $r_0 = 0.01$ (Blue)}
	\label{fig:discounts}
\end{figure}

\section{Anytime Proxy}
\label{app:anytime_proxy}
In this section, we provide the derivations for the ``anytime proxy'' relationships. In the case of exponential discounting, we recover~\eqref{eq:anytime_proxy} as follows:
\begin{align}
	C^{\gamma}_0 + G^{\gamma}_0 &= C^{\gamma}_0 + \sum_{k=0}^{\infty} \gamma^{k} R_{k+1} \nonumber\\
	& = C^{\gamma}_0 + \sum_{k=0}^{t-1} \gamma^k R_{k+1} + \sum_{k=t}^{\infty} \gamma^k R_{k+1} \nonumber\\
	& = C^{\gamma}_0 + \sum_{k=0}^{t-1} \gamma^k R_{k+1} + \gamma^t \sum_{k=0}^{\infty} \gamma^k R_{t+k+1} \nonumber\\
	& = \gamma^t \left( \frac{C^{\gamma}_0 + \sum_{k=0}^{t-1} \gamma^k R_{k+1}}{\gamma^t} +  \sum_{k=0}^{\infty} \gamma^k R_{t+k+1} \right) \nonumber\\
	& = \gamma^t \left( C^{\gamma}_t +  G^{\gamma}_t \right).
\end{align}

Equation~\eqref{eq:anytime_proxy_d} for general discounting is given by
\begin{align}
	C^d_0 + G^d_0 &= C^d_0 + \sum_{k=0}^{T-1}  d_{k} R_{k+1} \nonumber \\
	& = \left(C^d_0 + \sum_{k=0}^{t-1} d_{k} R_{k+1}\right) + \sum_{k=t}^{T-1} d_{k} R_{k+1} \nonumber \\
	& = d_{t}C^d_t + d_{t} \sum_{k=0}^{T-t-1} \frac{d(t+k)}{d_{t}} R_{t+k+1} \nonumber \\
	& = d_{t} \left(C^d_t + G^d_t\right).
\end{align}

For~\eqref{eq:anytime_proxy_m} as described in the multi-horizon setting, we have
\begin{align}
	C^{\tilde{d}}_0 + G^{\tilde{d}}_0 &= \tilde{d}_tC^{\tilde{d}}_t + \tilde{d}_t \sum_{k=0}^{T-t-1} \frac{\tilde{d}(t+k)}{\tilde{d}_t} R_{t+k+1} \nonumber \\
	& = \tilde{d}_tC^{\tilde{d}}_t + \tilde{d}_t \sum_{k=0}^{T-t-1} \frac{\sum_{i=1}^m w_i \gamma_i^{t+k}}{\tilde{d}_t} R_{t+k+1} \nonumber \\
	& = \tilde{d}_tC^{\tilde{d}}_t + \tilde{d}_t \sum_{i=1}^m  \frac{w_i \gamma_i^{t}}{\tilde{d}_t} \sum_{k=0}^{T-t-1}  \gamma_i^{k} R_{t+k+1} \nonumber \\
	& = \tilde{d}_t \left(C^{\tilde{d}}_t + \sum_{i=1}^m w_{i,t} G^{\gamma_i}_t\right).
\end{align}

\section{OCE Utility Functions}
\label{app:risk_measures}

\begin{figure}[!htb]
	\centerline{\includegraphics[width=0.6\linewidth]{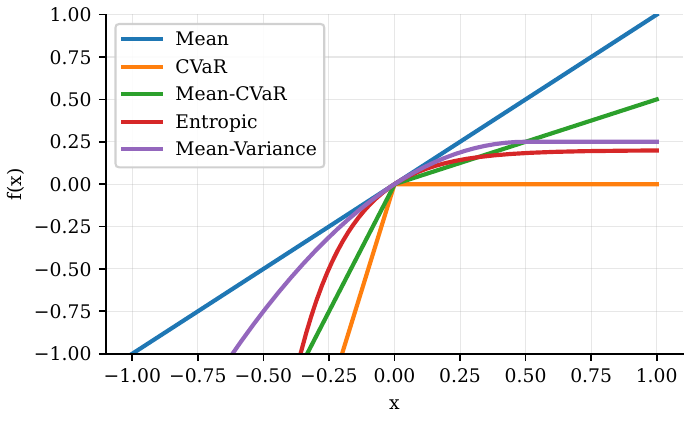}}
	\caption{\textbf{Utility functions for common OCE risk measures.} Detailed formulations are provided in Table~\ref{tab:risk_measures}.}
	\label{fig:utilities}
\end{figure}

The formulas for utility functions depicted in Figure \ref{fig:utilities} are presented in Table \ref{tab:risk_measures}.
\begin{table}[!ht]
	\caption{Common Risk Measures as OCEs.}
	\label{tab:risk_measures}
	\begin{center}
		\begin{small}
			\begin{sc}
				\begin{tabular}{l l l l}
					\toprule
					Risk Name & Parameter & Formula & Utility function $f(x)$\\
					\midrule
					Mean & None & $\mathbb{E} [G]$ & $x$ \\
					\hline CVaR & $\tau \in(0,1]$ & $\operatorname{CVaR}_\tau(G)$ & $\tau^{-1}(x)_{-}$ \\
					\hline \multirow{2}{*}{Mean-CVaR} & $\kappa_1 \in[0,1]$ & \multirow{2}{*}{$\kappa_1 \mathbb{E} [G]+\left(1-\kappa_1\right) \cdot \mathrm{CVaR}_\tau(G)$} & $\kappa_1(x)_{+}+\kappa_2(x)_{-},$ where \\
					& $\tau \in(0,1]$ &  & $\kappa_2=\tau^{-1}\left(1-\kappa_1\right)+\kappa_1$ \\
					\hline Entropic Risk & $\beta \in (0, \infty)$ & $\operatorname{Entr}_\beta(G)$ & $\frac{1}{\beta}(1 - \exp (-\beta x))$ \\
					\hline \multirow{2}{*}{Mean-Variance} & \multirow{2}{*}{$\kappa \in (0, \infty)$} & \multirow{2}{*}{$\mathbb{E} [G] - \kappa \cdot \operatorname{Var}(G)$} & $(x-\kappa x^2)\mathbb{I}\{x \leq \frac{1}{2\kappa}\}$ \\
					& & & $\quad + \frac{1}{4\kappa}\mathbb{I}\{x > \frac{1}{2\kappa}\}$ \\
					\bottomrule
				\end{tabular}
			\end{sc}
		\end{small}
	\end{center}
\end{table}

\section{Proofs in Finite-Horizon Setting}
\label{app:finite}
First, we recall the properties from~\citet{Pires.etal2025}. In what follows, we denote the 1-Wasserstein and supremum 1-Wasserstein distances between $\nu,\nu'\in\Delta(\reals^n)$ by $\wass(\nu, \nu')$ and $\overline{\wass}(\nu,\nu')$, respectively.
\begin{definition}[Indifference to Scaling]
	\label{def:indif-scaling}
	For all distributions $\nu, \nu' \in (\Delta(\bs), \wass)$ such that $K\nu \geq K\nu'$, then $K(\gamma G) \geq K(\gamma G')$, where $G \sim \nu, G' \sim \nu'$ and $K(G) = K \df (G)$.
\end{definition}

\begin{definition}[Indifference to Mixtures]
	\label{def:indif-mixture}
	For every $\eta, \eta' \in (\Delta(\bs)^{\states \times \mathcal{C}}, \overline{\wass})$ such that $K\eta(s, \xc) \geq K\eta'(s, \xc)$ for all $(s, \xc) \in \states \times \mathcal{C}$, then for any random variable $(S, \xC)$ taking values in $\states \times \mathcal{C}$ we also have $K(G(S, \xC)) \geq K(G'(S, \xC))$, where $G(s, \xc) \sim \eta(s, \xc)$ and $G'(s, \xc) \sim \eta'(s, \xc)$.
\end{definition}

\subsection{Time-Dependent Monotonicity Lemma}
\label{app:finite1}
\begin{lemma}[Time-Dependent Monotonicity (Restated)]
	Let $K$ be an objective functional satisfying the indifference to mixture property. For any time $t < T$, let $\eta_{t+1}$ and $\eta'_{t+1}$ be two return distribution functions for time $t+1$. If $F_K \eta_{t+1} \geq F_K \eta'_{t+1}$, then for any time-$t$ policy $\pi_t$:
	$F_K (\mathcal{T}^d_{\pi_t} \eta_{t+1}) \geq F_K (\mathcal{T}^d_{\pi_t} \eta'_{t+1})$.
\end{lemma}

\begin{proof}[Proof]
	The proof adapts Lemma 14 in~\citet{Pires.etal2025}. We fix a policy $\pi_t$ and state $(s,c)\in \states \times \mathcal{C}$. The hypothesis $F_K \eta_{t+1} \geq F_K \eta'_{t+1}$ implies that for any realization of the next state and stock $(s', c')$, we have $K\df(d_{t+1}c' + d_{t+1}G) \geq K\df(d_{t+1}c' + d_{t+1}G')$, where $G \sim \eta_{t+1}(s', c')$ and $G' \sim \eta'_{t+1}(s', c')$. 
	
	Using the recursive definitions of the stock update and the Bellman operator, the total unscaled outcome at time $t$ satisfies the identity:
	\[
	d_t c + d_t (R_{t+1} + \hat{d}_t G) \distequiv d_{t+1}C^d_{t+1} +d_{t+1}G.
	\]
	This reveals that the random variable evaluated by $K$ at time $t$ is a probabilistic mixture (over $S_{t+1}, R_{t+1}$) of the random variables evaluated at time $t+1$. Since the preference holds component-wise for every realization, the Indifference to Mixtures property in Definition~\ref{def:indif-mixture} ensures the preference is preserved for the mixture, proving the lemma.
\end{proof}

\subsection{Optimality of Distributional Backward Induction}
\label{app:finite2}
\begin{theorem}[Optimality of Distributional Backward Induction (Restated)]
	Let the problem horizon $T\in\mathbb{N}$ be finite and known, and let the objective functional $K$ be indifferent to mixtures. Assume that the sequence of return distributions $\eta^*=(\eta^*_0, \dots, \eta^*_T)$ is generated by the following backward induction algorithm:
	\begin{itemize}
		\item \textbf{Base Case:} For all $(s,c) \in \states \times \mathcal{C}$, set $\eta^*_T(s,c) = \delta_0$.
		\item \textbf{Recursive Step:} For $t = T-1$ down to $0$, compute $\eta^*_t \doteq \mathcal{T}^d_* \eta^*_{t+1}$,
		where $\mathcal{T}^d_*$ is the \emph{greedy time-dependent distributional Bellman operator}, defined such that for any function $\eta_{t+1}$ and for all $(s,c)\in \states \times \mathcal{C}$, it returns the function $\eta_t$ satisfying:
		\[
		(F_K \eta_t)(s,c) = \sup_{\pi_t} (F_K (\mathcal{T}^d_{\pi_t} \eta_{t+1}))(s,c).
		\]
	\end{itemize}
	Then, the sequence $\eta^*$ is a sequence of optimal return distributions. Furthermore, the non-stationary policy $\pi^* = (\pi^*_0, \dots, \pi^*_{T-1})$, where each $\pi^*_t$ is a \emph{greedy policy} with respect to $\eta^*_{t+1}$, is optimal.    
\end{theorem}

\begin{proof}[Proof]
	The proof proceeds by backward induction on the time step $t$.
	\textbf{Base Case:} The theorem holds trivially at $t=T$. At the horizon, $G^d_T=0$, so the optimal return from any state $(s,c)$ and time $T$ is deterministically zero, hence $\eta^*_T(s,c) = \delta_0$ is the optimal return distribution function.
	\textbf{Recursive Step:} Assume that for some time step $t+1 \in \{1, \dots, T\}$, $\eta^*_{t+1}$ is the optimal return distribution function for any subproblem starting at or after time $t+1$. We aim to show that $\eta^*_t \doteq \mathcal{T}^d_* \eta^*_{t+1}$ is optimal at time $t$. Let $\pi^*_t$ be a greedy policy w.r.t. $\eta^*_{t+1}$. Consider an arbitrary policy $\pi_{\geq t} = (\pi_t, \pi_{\geq t+1})$. Its return distribution is $\eta^{\pi_{\geq t}} = \mathcal{T}^d_{\pi_t} \eta^{\pi_{\geq t+1}}$. By the inductive hypothesis, $F_K \eta^*_{t+1} \geq F_K \eta^{\pi_{\geq t+1}}$. Applying the Time-Dependent Monotonicity Lemma gives $F_K (\mathcal{T}^d_{\pi_t} \eta^*_{t+1}) \geq F_K (\mathcal{T}^d_{\pi_t} \eta^{\pi_{\geq t+1}}) = F_K \eta^{\pi_{\geq t}}$. By definition of the greedy operator, $F_K \eta^*_t \geq F_K (\mathcal{T}^d_{\pi_t} \eta^*_{t+1})$. Chaining these inequalities yields $F_K \eta^*_t \geq F_K \eta^{\pi_{\geq t}}$. Since $\pi_{\geq t}$ was arbitrary, $\eta^*_t$ is optimal.
\end{proof}

\subsection{OCE Optimality Bound}
\label{app:finite3}

\begin{proposition}[OCE Optimality Bound (Restated)]
	Let $f: \mathbb{R} \to \mathbb{R}$ be an $L$-Lipschitz utility function. Let $\pi^*$ be the optimal policy from Theorem~\ref{thm:backward-induction-optimality}, and let $\hat{\pi}$ be the policy derived from an approximate backward induction with one-step distributional projection errors $(\varepsilon_0, \dots, \varepsilon_{T-1})$. Let the outer optimization be performed over a finite grid $\bar{\mathcal{C}}$ with spacing $\delta$. If $\hat{c} \in \bar{\mathcal{C}}$ is the initial stock selected by the algorithm, the total suboptimality is bounded by:
	\[
	\operatorname{OCE}_f(\eta^*_0) - \left( -\hat{c} + \mathbb{E}[f(\hat{c} + G^{\hat{\pi}})] \right) \leq 2 \mathcal{E}_{\text{DP}} + \frac{(1 + L)\delta}{2},
	\]
	where $\mathcal{E}_{\text{DP}} \doteq L \sum_{t=0}^{T-1} d_t \varepsilon_t$ is the cumulative DP error term. 
\end{proposition}

\begin{proof}
	The proof proceeds in three steps: deriving the inner approximation error, establishing the properties of the outer objective, and combining them into a total error bound.
	
	\textbf{Step 1 -- Inner DP Error Derivation.}
	Let $\hat{\eta}_t = \mathcal{A}(\mathcal{T}^d_* \hat{\eta}_{t+1})$ be the return distribution computed by the approximate backward induction, where $\mathcal{A}$ is an approximation operator and $\mathcal{T}^d_*$ uses the same greedy tie-breaking policy as the policy $\hat{\pi}_t$ returned by the algorithm. Define the one-step projection error
	\[
	\varepsilon_t \doteq \sup_{s,c} \wass((\mathcal{T}^d_* \hat{\eta}_{t+1})(s,c), \hat{\eta}_t(s,c)).
	\]
	We analyze the propagation of this error directly in the scalar utility-value space. From the definition of $U_f$ and the stock update, for any fixed time-$t$ policy $\pi_t$ and any two next-step return distribution functions $\eta_{t+1}$ and $\eta'_{t+1}$,
	\[
	\|U_f(\mathcal{T}^d_{\pi_t}\eta_{t+1})-U_f(\mathcal{T}^d_{\pi_t}\eta'_{t+1})\|_\infty
	\leq \hat{d}_t\|U_f\eta_{t+1}-U_f\eta'_{t+1}\|_\infty .
	\]
	Taking the supremum over policies is non-expansive for scalar values, thus
	\begin{align*}
		&\|U_f(\mathcal{T}^d_*\eta_{t+1})-U_f(\mathcal{T}^d_*\eta'_{t+1})\|_\infty \\
		&\quad \leq \sup_{\pi_t}\|U_f(\mathcal{T}^d_{\pi_t}\eta_{t+1})-U_f(\mathcal{T}^d_{\pi_t}\eta'_{t+1})\|_\infty
		\leq \hat{d}_t\|U_f\eta_{t+1}-U_f\eta'_{t+1}\|_\infty .
	\end{align*}
	Thus, with $\Delta_t \doteq \|U_f\eta^*_t-U_f\hat{\eta}_t\|_\infty$ and $\Delta_T=0$,
	
	\begin{align*}
		\Delta_t
		&\leq \|U_f(\mathcal{T}^d_*\eta^*_{t+1})-U_f(\mathcal{T}^d_*\hat{\eta}_{t+1})\|_\infty
		+\|U_f(\mathcal{T}^d_*\hat{\eta}_{t+1})-U_f\hat{\eta}_t\|_\infty \\
		&\leq \hat{d}_t \Delta_{t+1}+L\varepsilon_t .
	\end{align*}
	The second inequality combines the scalar greedy bound above with the $L$-Lipschitz property of $U_f$ with respect to the 1-Wasserstein distance $\wass$; the latter follows since $x\mapsto f(d_t c+d_t x)/d_t$ is $L$-Lipschitz. Unrolling the recursion gives
	\[
	\|U_f\eta^*_0-U_f\hat{\eta}_0\|_\infty \leq L\sum_{t=0}^{T-1}\left(\prod_{k=0}^{t-1}\hat{d}_k\right)\varepsilon_t
	= L\sum_{t=0}^{T-1}d_t\varepsilon_t
	= \mathcal{E}_{\text{DP}},
	\]
	where the empty product is one and $d_0=1$.
	Furthermore, we must account for the performance of the policy $\hat{\pi}$ actually derived from these approximate distributions. Let $\eta^{\hat{\pi}}_t$ denote its true return distribution. Since $\hat{\pi}_t$ is greedy with respect to $\hat{\eta}_{t+1}$ under the algorithm's tie-breaking, $\mathcal{T}^d_*\hat{\eta}_{t+1}=\mathcal{T}^d_{\hat{\pi}_t}\hat{\eta}_{t+1}$ in utility value. Applying the fixed-policy scalar Bellman bound and the same projection error yields
	\[
	\|U_f\hat{\eta}_t-U_f\eta^{\hat{\pi}}_t\|_\infty
	\leq L\varepsilon_t+\hat{d}_t\|U_f\hat{\eta}_{t+1}-U_f\eta^{\hat{\pi}}_{t+1}\|_\infty .
	\]
	Unrolling from the terminal condition gives $\|U_f\hat{\eta}_0-U_f\eta^{\hat{\pi}}_0\|_\infty \leq \mathcal{E}_{\text{DP}}$.
	
	\textbf{Step 2 -- Outer Objective Properties.}
	Let $J(c) \doteq -c + (U_f \eta^*_0)(s_0, c)$ be the true outer objective and $\hat{J}(c) \doteq -c + (U_f \hat{\eta}_0)(s_0, c)$ be the approximate objective used by the algorithm. From Step 1, we have the uniform bound $\sup_c |J(c) - \hat{J}(c)| \leq \mathcal{E}_{\text{DP}}$.
	We also establish the Lipschitz continuity of the outer objective. Since $f$ is $L$-Lipschitz, the expected utility term is $L$-Lipschitz with respect to $c$. Including the linear term $-c$, $J(c)$ is thus $(1+L)$-Lipschitz.
	
	\textbf{Step 3 -- Total Error Decomposition.}
	Let $c^*$ be the global maximizer of $J(c)$ (i.e., the true optimal stock) and $\hat{c}$ be the maximizer of $\hat{J}(c)$ on the grid $\bar{\mathcal{C}}$ (i.e., the selected stock). Let $c^*_{grid}$ be the point in $\bar{\mathcal{C}}$ closest to $c^*$, such that $|c^* - c^*_{grid}| \leq \delta/2$.
	We decompose the total optimality gap, noting that the true value of the learned policy is $J^{\hat{\pi}}(\hat{c}) \doteq - \hat{c} + \mathbb{E}[f(\hat{c} + G^{\hat{\pi}})]$:
	\begin{align*}
		\text{Gap} &= J(c^*) - J^{\hat{\pi}}(\hat{c}) \\
		&= \underbrace{J(c^*) - J(c^*_{grid})}_{\text{Discretization}} + \underbrace{J(c^*_{grid}) - \hat{J}(c^*_{grid})}_{\text{Estimation Error}} + \underbrace{\hat{J}(c^*_{grid}) - \hat{J}(\hat{c})}_{\text{Optimization}} + \underbrace{\hat{J}(\hat{c}) - J^{\hat{\pi}}(\hat{c})}_{\text{Policy Performance Gap}} \\
		&\leq (1+L)\frac{\delta}{2} + \mathcal{E}_{\text{DP}} + 0 + \mathcal{E}_{\text{DP}} \\
		&= 2\mathcal{E}_{\text{DP}} + (1+L)\frac{\delta}{2}.
	\end{align*}
	The third term is $\leq 0$ because $\hat{c}$ maximizes $\hat{J}$ over the grid. The fourth term applies the bound derived in Step 1 to the specific stock $\hat{c}$.
\end{proof}

\section{Infinite-Horizon Setting}
\label{app:infinite}
\subsection{Modified Integral Representation}
\label{app:modified_integral_representation}

A second method is to modify the integral representation of the discount function. As shown in~\citet{Fedus.etal2019a}, some discount functions, including hyperbolic, can be expressed as an integral of exponential functions over the discount factor $\gamma \in [0,1]$:
\[
d_t = \int_{\gamma=0}^{1} w(\gamma) \gamma^{t}  d\gamma.
\]
For hyperbolic discounting, $w(\gamma)=\frac{1}{k}\gamma^{1/k-1}$ . The convergence of $\hat{d}_{t}$ to 1 is a direct consequence of the integration limit including $\gamma=1$. By simply reducing the upper limit of integration to $\gamma_{\text{tail}} < 1$, we can construct a new discount function that excludes this problematic point. Solving this integral yields:
\[
\int_{\gamma=0}^{\gamma_{\text{tail}}} \frac{1}{k}\gamma^{1/k+t-1} d\gamma = \left[ \frac{\gamma^{1/k+t}}{1+kt} \right]_0^{\gamma_{\text{tail}}} = \frac{\gamma_{\text{tail}}^{1/k+t}}{kt+1}.
\]
To ensure the standard normalization $d_0=1$, we can normalize this function by its value at $t=0$, which is $\gamma_{\text{tail}}^{1/k}$. This gives the final discount function:
\[
d^{\text{new}}_t = \frac{\gamma_{\text{tail}}^{t}}{1+kt}.
\]
This function retains a hyperbolic-like character due to the $1/(kt+1)$ term, but its asymptotic behavior is now governed by the exponential term. The one-step discount factor is $\hat{d}^{\text{new}}_{t} = \gamma_{\text{tail}} \cdot \frac{kt+1}{kt+k+1}$, and its limit is $\gamma_{\text{tail}}$.

\subsection{Asymptotic Risk-Neutrality}
\label{app:asymptotic_risk_neutrality}

\begin{lemma}[Asymptotic Risk-Neutrality (Restated)]
	Let $f$ be a strictly increasing, twice continuously differentiable utility function with a bounded second derivative and $f'>0$ on the relevant compact outcome interval, and assume rewards are uniformly bounded. Then, as $t \to \infty$, the risk-sensitive action values converge uniformly to the risk-neutral action values. Consequently, the smooth utility objective becomes asymptotically risk-neutral in its action ordering. 
\end{lemma}

\begin{proof}[Proof of Lemma~\ref{lem:asymptotic-neutrality}]
	From~\eqref{eq:anytime_proxy_d}, the total return can be expressed as $d_t c_t + d_t G_t$. The term $d_t c_t = C_0 + \sum_{k=0}^{t-1} d_k R_{k+1}$ represents the unscaled accumulated rewards up to time $t-1$ and is constant with respect to the action $a_t$. The decision-dependent term is $d_t G_t$.
	
	The agent's objective is to maximize $\mathbb{E}[f(d_t c_t + d_t G_t)]$. Let $R_{\max}$ be such that $|R_{k+1}|\le R_{\max}$, and define $D_t=\sum_{k=t}^{\infty}d_k$. Then $|G_t|\le R_{\max}D_t/d_t$, so $|d_tG_t|\le R_{\max}D_t$. Since $d_t\to0$ and $\sum_t d_t<\infty$, we have $D_t\to0$. The total outcome $d_tc_t+d_tG_t$ remains in a compact interval; hence there exist constants $M<\infty$ and $m>0$ on this interval such that $|f''|\le M$ and $f'\ge m$.
	
	For any action $a\in\actions$, Taylor's theorem around $d_t c_t$ gives
	\[
	f(d_t c_t+d_tG_t^a)
	=
	f(d_t c_t)+f'(d_t c_t)d_tG_t^a+\mathcal{R}_t^a,
	\]
	where
	\[
	|\mathcal{R}_t^a|
	\le
	\frac{M}{2}(d_tG_t^a)^2 .
	\]
	Taking expectations and using $|G_t^a|\le R_{\max}D_t/d_t$, we get the uniform bound
	\[
	|\E[\mathcal{R}_t^a]|
	\le
	\frac{M}{2}R_{\max}^2D_t^2 .
	\]
	Normalizing by the positive quantity $D_tf'(d_tc_t)$ yields
	\[
	\frac{\E[f(d_tc_t+d_tG_t^a)]-f(d_tc_t)}{D_tf'(d_tc_t)}
	=
	\frac{d_t}{D_t}\E[G_t^a]
	+
	\varepsilon_t(a),
	\]
	with
	\[
	|\varepsilon_t(a)|
	\le
	\frac{M R_{\max}^2}{2m}D_t
	\to 0 .
	\]
	This convergence is uniform over actions. Since $d_t/D_t$ is positive and independent of the action, maximizing the normalized first-order term is equivalent to maximizing the risk-neutral value $\E[G_t^a]$. Consequently, the second-order utility correction vanishes relative to the tail scale, and the smooth utility objective becomes asymptotically risk-neutral in its action ordering.
\end{proof}

\subsection{Performance Bound for Approximate Infinite-Horizon Policy}
\label{app:performance_bound}

Our analysis is inspired by the work of \citet{Hau.etal2023b} on entropic risk measures. However, their approach leverages the specific factorization properties of the entropic measure. Our framework considers more general discount functions using distributional value functions and the stock-augmented space. The performance bound below is proved for the entropic certainty equivalent; non-smooth objectives such as Mean-CVaR are discussed separately after Lemma~\ref{lem:mean_cvar_lower_bound}.

For the derivation of the performance bound, we utilize the specific certainty-equivalent form $V_t(s, c)$ defined below. We note that under the assumption that the utility function $f$ is strictly increasing, optimizing this formulation yields the exact same greedy policies as the objective $F_K \eta^\pi_t$. This is because the inverse function $f^{-1}$ is strictly monotonic, and the subsequent subtraction of $d_t c$ and scaling by $1/d_t$ represent affine transformations that are constant with respect to the decision at time $t$. Thus, the preference ordering over actions remains invariant.

\begin{lemma}[Lower Bound for the Certainty Equivalent Value Function]
	\label{lem:performance_bound}
	Let the value function at time $t$ for a state $s\in\states$ and scaled stock $c\in\mathcal{C}$ be defined as
	\[
	V_t(s, c) \doteq \frac{1}{d_t} \left( f^{-1}\left(\E\left[f(d_t c + d_t G_t(s, c))\right]\right) - d_t c \right) ,
	\]
	where $d_t > 0$ is the discount factor, and $G_t(s, c)$ is a bounded random variable representing future returns. Let $f$ be defined as $f(x) = \beta^{-1}(1 - e^{-\beta x})$ for $\beta>0$, $D_t \doteq \sum_{k=t}^{\infty} d_k$, and $\Delta_R \doteq R_{\max} - R_{\min}$. Then, the value function satisfies
	\[
	V_t(s, c) \ge \E[G_t(s, c)] - \frac{\beta \Delta_R^2 D_t^2}{8d_t}.
	\]
\end{lemma}

\begin{proof}
	Let $X = d_t c + d_t G_t(s, c)$ be the random variable representing the total accumulated value, and let $\mu_X = \E[X]$. Since $G_t(s,c)=d_t^{-1}\sum_{k=t}^{\infty} d_k R_{k+1}$ and the reward support has range $\Delta_R$, the support range of $G_t(s,c)$, denoted by $\Delta_{G,t}$, is at most
	\[
	\Delta_{G,t} \le \frac{\Delta_R}{d_t}D_t.
	\]
	Consequently, the support range of $X$ is bounded by
	\[
	\Delta_X \le d_t \Delta_{G,t} \le \Delta_R D_t.
	\]
	
	We start from the certainty equivalent
	\[
	\operatorname{CE}_f(X) = f^{-1}(\E[f(X)]) = -\frac{1}{\beta}\log \E\left[ e^{-\beta X} \right] = \mu_X - \frac{1}{\beta}\log \E\left[ e^{-\beta (X-\mu_X)} \right].
	\]
	Hoeffding's lemma applied to the centered bounded random variable $X-\mu_X$ gives
	\[
	\log \E\left[ e^{-\beta (X-\mu_X)} \right] \le \frac{\beta^2 \Delta_X^2}{8},
	\]
	and therefore
	\[
	\operatorname{CE}_f(X) \ge \mu_X - \frac{\beta \Delta_X^2}{8}.
	\]
	We can then substitute $\operatorname{CE}_f$ into the definition of $V_t(s, c)$ to bound the value function:
	\begin{align*}
		V_t(s, c) &\overset{(a)}{\ge} \frac{\left(\mu_X - \frac{1}{8} \beta \Delta_X^2\right) - d_t c}{d_t} \\
		&\overset{(b)}{=} \E[G_t(s, c)] - \frac{\beta \Delta_X^2}{8d_t} \\
		&\overset{(c)}{\ge} \E[G_t(s, c)] - \frac{\beta \Delta_R^2 D_t^2}{8d_t}.
	\end{align*}
	\textbf{(a)} follows from applying the previous inequality to $V_t(s, c) = \frac{1}{d_t}(\operatorname{CE}_f(X) - d_t c)$. \textbf{(b)} comes from $\mu_X = d_t c + d_t \E[G_t(s, c)]$. \textbf{(c)} uses the support range bound $\Delta_X \le \Delta_R D_t$.
\end{proof}

Using this lemma, we can derive a performance bound for the infinite-horizon optimal policy using the approximation scheme described in Algorithm~\ref{alg:infinite-horizon-approx}. The following proof provides a bound for the performance of the policy returned by Algorithm~\ref{alg:infinite-horizon-approx} for any stock $c\in\mathcal{C}$, corresponding to the DP errors of the inner optimization. To relate this bound to the OCE risk measure, one must perform the outer optimization via a grid search and approximately solve the OCE, in a similar manner to Steps 2 and 3 of the proof of Proposition~\ref{prop:bound-approx-backward-induction}. The performance analysis then depends on both the grid range and its level of discretization.

\begin{theorem}[Performance Bound for Approximate Infinite-Horizon Policy (Restated)]
    Let the objective functional $K$ be the \emph{Entropic Risk Measure} with risk aversion coefficient $\beta > 0$ (corresponding to the exponential utility $f(x) = \beta^{-1}(1-e^{-\beta x})$). Let the discount function $d_{t}$ follow the truncation scheme with parameter $\gamma_{\text{tail}} < 1$ for $t \ge T'$, and define $D_t \doteq \sum_{k=t}^\infty d_k$. Let $\pi^*$ be the optimal non-stationary policy for the infinite-horizon problem and $\hat{\pi}^*$ the policy returned by Algorithm~\ref{alg:infinite-horizon-approx} with planning horizon $T'$. Then, writing $V^{\pi,\mathrm{CE}}_0$ for the entropic certainty-equivalent value of policy $\pi$ at time $t=0$, the suboptimality of the approximate policy is bounded by
	\[
	\| V^{*,\mathrm{CE}}_0 - V^{\hat{\pi}^*,\mathrm{CE}}_0 \|_\infty \le \frac{\beta \Delta_R^2 D_{T'}^2}{8} .
	\]
\end{theorem}

\begin{proof}[Proof of Theorem~\ref{thm:performance_bound}]
	
	\subsection*{Step 1 -- Bounding the Error at the Truncation Horizon}

    We first establish a bound on the difference between the risk-neutral value and the CE value for any policy $\pi$. Let $V^{\pi, \text{RN}}_t(s, c) \doteq \E[G_t^\pi(s, c)]$ denote the discounted risk-neutral value function (i.e., expected return), and let $V^{\pi, \text{CE}}_t(s, c)$ denote the discounted CE value function. Using the lower bound for the certainty equivalent value function from Lemma~\ref{lem:performance_bound}, for any policy $\pi$, time $t$, state $s\in\states$, and stock $c\in\mathcal{C}$, we have $V^{\pi, \text{CE}}_t(s, c) \ge V^{\pi, \text{RN}}_t(s, c) - \mathcal{E}_t$, where we denote $\mathcal{E}_{t} = \frac{\beta \Delta_R^2 D_t^2}{8d_t}$. Conversely, by the concavity of the utility function $f$~\citep[Proposition 2.2 in][]{Ben-Tal.Teboulle2007}, the CE is upper-bounded by the expectation. Thus,
	\begin{equation}\label{eq:general_diff_bound}
		0 \overset{(a)}{\le} V^{\pi, \text{RN}}_t(s, c) - V^{\pi, \text{CE}}_t(s, c) \overset{(b)}{\le} \mathcal{E}_t,
	\end{equation}
	where \textbf{(a)} holds because $\E$ is an upper-bound on the CE, and \textbf{(b)} follows from Lemma~\ref{lem:performance_bound}.
	
	We now bound the difference between the optimal value function of the infinite-horizon discounted CE-MDP, $V^{*, \text{CE}}_{T'}$, and the value function returned by Algorithm~\ref{alg:infinite-horizon-approx}, $V^{\hat{\pi}^*, \text{CE}}_{T'}$, at the planning horizon $T'$ and for any state and stock $(s,c)$ as follows:
	\begin{align*}
		V^{*, \text{CE}}_{T'}(s, c) - V^{\hat{\pi}^*, \text{CE}}_{T'}(s, c) 
		&\overset{(a)}{\le} V^{*, \text{CE}}_{T'}(s, c) - V^{\hat{\pi}^*, \text{RN}}_{T'}(s, c) + \mathcal{E}_{T'} \\
		&\overset{(b)}{\le} V^{\pi^*, \text{RN}}_{T'}(s, c) - V^{\hat{\pi}^*, \text{RN}}_{T'}(s, c) + \mathcal{E}_{T'} \\
		&\overset{(c)}{\le} \mathcal{E}_{T'}.
	\end{align*}
	Here, \textbf{(a)} follows from the RHS of~\eqref{eq:general_diff_bound} for $\hat{\pi}^*$.  \textbf{(b)} comes from the fact that using the LHS of~\eqref{eq:general_diff_bound} for the optimal CE policy $\pi^*$, we have $0 \le V^{\pi^*, \text{RN}}_{T'} - V^{\pi^*, \text{CE}}_{T'} \implies V^{*, \text{CE}}_{T'} \le V^{\pi^*, \text{RN}}_{T'}$.
	\textbf{(c)} is true because $\hat{\pi}^*_{T'} = \pi^*_{\text{RN}} \in \argmax_{\pi} V^{\pi, \text{RN}}_{T'}$ is the policy that maximizes the expected return. Thus, $V^{\pi^*, \text{RN}}_{T'}(s, c) \le V^{\hat{\pi}^*, \text{RN}}_{T'}(s, c)$, making the difference non-positive for all $(s,c)\in\states\times\mathcal{C}$.
	
	\subsection*{Step 2 -- Constructing a Lower Bound for the Value of the Approximate Policy}
	
	We construct a value function $U_t$ and a corresponding return distribution function $\eta^{\text{LB}}_t$ for all $t \in 0:T'$ which serves as a lower bound for the true CE value of the algorithm's policy $V^{\hat{\pi}^*, \text{CE}}_t$. This construction proceeds by backward induction.
	
	\textbf{Base Case ($t=T'$):}
	The algorithm defines the policy $\hat{\pi}^*$ such that for $t \ge T'$, $\hat{\pi}^*_t = \pi^*_{\text{RN}}$. Therefore, the true CE value of the algorithm's policy at the horizon is the CE value of the risk-neutral policy. In addition, the algorithm initializes its computation with the risk-neutral value: $\hat{V}^*_{T'}(s, c) = V^{\pi^*_{\text{RN}}, \text{RN}}_{T'}(s, c)$.
	Using the bound established in~\eqref{eq:general_diff_bound} of Step 1 applied to the policy $\pi = \pi^*_{\text{RN}}$, we have:
	\[
	V^{\hat{\pi}^*, \text{CE}}_{T'}(s, c) = V^{\pi^*_{\text{RN}}, \text{CE}}_{T'}(s, c) \ge V^{\pi^*_{\text{RN}}, \text{RN}}_{T'}(s, c) - \mathcal{E}_{T'}.
	\]
	We define the lower-bound return distribution $\eta^{\text{LB}}_{T'}$ by shifting the algorithm's terminal distribution by the constant error $\mathcal{E}_{T'}$:
	\[
	\eta^{\text{LB}}_{T'}(s, c) \doteq \text{shift}(\hat{\eta}^*_{T'}(s, c), -\mathcal{E}_{T'}).
	\]
	Let $U_{T'}(s, c) = F_K(\eta^{\text{LB}}_{T'})(s, c)$. By the translation invariance of $F_K$, we have
	\[
	U_{T'}(s, c) = F_K(\hat{\eta}^*_{T'})(s, c) - \mathcal{E}_{T'} = \hat{V}^*_{T'}(s, c) - \mathcal{E}_{T'}.
	\]
	Combining these yields the base case inequality
	\[
	V^{\hat{\pi}^*, \text{CE}}_{T'}(s, c) \ge U_{T'}(s, c).
	\]
	
	\textbf{Recursive Step ($t < T'$):}
	We define $U_t$ and the corresponding $\eta^{\text{LB}}_t$ recursively using the Bellman operator for the algorithm's policy $\hat{\pi}^*_t$:
	\[
	U_t = F_K \eta^{\text{LB}}_t \quad \mbox{and} \quad \eta^{\text{LB}}_t = \mathcal{T}^d_{\hat{\pi}^*_t} \eta^{\text{LB}}_{t+1}.
	\]
	Assume that for some $t+1 \leq T'$, the hypothesis $V^{\hat{\pi}^*, \text{CE}}_{t+1}(s, c) \ge U_{t+1}(s, c)$ holds. We aim to show that $V^{\hat{\pi}^*, \text{CE}}_t(s, c) \ge U_t(s, c)$ is true at time $t$.
	By definition, $V^{\hat{\pi}^*, \text{CE}}_{t+1} = F_K \eta^{\hat{\pi}^*}_{t+1}$ and $U_{t+1} = F_K \eta^{\text{LB}}_{t+1}$. Thus, the hypothesis implies
	\[
	F_K \eta^{\hat{\pi}^*}_{t+1} \ge F_K \eta^{\text{LB}}_{t+1}.
	\]
	We now apply the Time-Dependent Monotonicity from Lemma~\ref{lem:time-monotonicity}. Since the objective functional is monotonic, applying the Bellman operator for the fixed policy $\hat{\pi}^*_t$ preserves the preference order:
	\[
	F_K (\mathcal{T}^d_{\hat{\pi}^*_t} \eta^{\hat{\pi}^*}_{t+1}) \ge F_K (\mathcal{T}^d_{\hat{\pi}^*_t} \eta^{\text{LB}}_{t+1}).
	\]
	The LHS is the definition of the true value $V^{\hat{\pi}^*, \text{CE}}_t(s, c)$, while the RHS is the definition of $U_t(s, c)$. Therefore, we recover the desired inequality
	\[
	V^{\hat{\pi}^*, \text{CE}}_t(s, c) \ge U_t(s, c).
	\]
	
	\textbf{Relationship between $\hat{\pi}^*$ and $U$:}
	Finally, we characterize the difference between the algorithm's value function $\hat{V}^*_t$ and the lower bound $U_t$. By construction, $\eta^{\text{LB}}_{T'}$ differs from $\hat{\eta}^*_{T'}$ by a constant shift of $-\mathcal{E}_{T'}$. Due to the structure of the Time-Dependent Bellman operator, where the future outcome is scaled by the one-step discount factor $\hat{d}_t$, a constant shift of $\Delta$ at time $t+1$ propagates as a shift of $\hat{d}_t\Delta$ at time $t$.
	Since $\hat{\pi}^*_t$ is used for both transitions, and $F_K$ is translation invariant, the error propagates deterministically backwards:
	\[
	U_t(s, c) = \hat{V}^*_t(s, c) - \left( \prod_{k=t}^{T'-1} \hat{d}_k \right) \mathcal{E}_{T'} .
	\]
	The term being subtracted depends only on $t$ and $T'$, not on the state $s$ or stock $c$. Since the greedy operator $\text{argmax}_\pi F_K \mathcal{T}^d_\pi V$ is invariant to adding a constant to the value function $V$, the policy $\hat{\pi}^*_t$ computed by the algorithm (which is greedy with respect to $\hat{V}^*_{t+1}$) is also greedy with respect to $U_{t+1}$.
	
	\subsection*{Step 3 -- Bounding the Distance to the True Optimal Value}
	
	We aim to show that for each state $s$, stock $c$, and time $t \in 0:T'$, the difference between the true optimal value and the lower bound constructed in Step 2 satisfies
	\begin{equation} \label{eq:step3_claim}
		V^{*, \text{CE}}_t(s, c) - U_t(s, c) \le \delta_t .
	\end{equation}
	where we define the sequence of bounds $\delta_t$ as
	\[
	\delta_t \doteq \left( \prod_{k=t}^{T'-1} \hat{d}_k \right) \mathcal{E}_{T'} .
	\]
	Note that this sequence satisfies the recurrence relation $\delta_t = \hat{d}_t \delta_{t+1}$ for $t < T'$, and $\delta_{T'} = \mathcal{E}_{T'}$. Also, note that $\prod_{k=t}^{T'-1} \hat{d}_k = \frac{d_{T'}}{d_{t}}$. We prove~\eqref{eq:step3_claim} by backward induction.
	
	\textbf{Base Case:} The inequality holds trivially for $t = T'$, because the risk-averse optimal value is upper-bounded by the risk-neutral optimal value, i.e., $V^{*, \text{CE}}_{T'} \le V^{*, \text{RN}}_{T'} = V^{\pi^*_{\text{RN}}, \text{RN}}_{T'}$, and thus:
	\begin{align*}
		V^{*, \text{CE}}_{T'}(s, c) - U_{T'}(s, c) 
		&= V^{*, \text{CE}}_{T'}(s, c) - \left( \hat{V}^*_{T'}(s, c) - \mathcal{E}_{T'} \right) \\
		&= V^{*, \text{CE}}_{T'}(s, c) - V^{\pi^*_{\text{RN}}, \text{RN}}_{T'}(s, c) + \mathcal{E}_{T'} \\
		&\le V^{\pi^*_{\text{RN}}, \text{RN}}_{T'}(s, c) - V^{\pi^*_{\text{RN}}, \text{RN}}_{T'}(s, c) + \mathcal{E}_{T'} \\
		&= \mathcal{E}_{T'} .
	\end{align*}
	
	\textbf{Recursive Step:} Assuming that the hypothesis~\eqref{eq:step3_claim} holds for $t+1$ for each $(s,c)\in\states\times\mathcal{C}$, we show that, for $t$ and all $(s,c)$, we obtain the following:
    \begin{align*}
		&V^{*, \text{CE}}_t(s, c) - U_t(s, c) \\ 
		&\quad \overset{(a)}{=} (F_K \mathcal{T}^d_{\pi^*_t} \eta^{*}_{t+1})(s, c) - (F_K \mathcal{T}^d_{\hat{\pi}^*_t} \eta^{\text{LB}}_{t+1})(s, c) \\
		&\quad \overset{(b)}{\le} (F_K \mathcal{T}^d_{\pi^*_t} \eta^{*}_{t+1})(s, c) - (F_K \mathcal{T}^d_{\pi^*_t} \eta^{\text{LB}}_{t+1})(s, c) \\
		&\quad \overset{(c)}{=} \frac{1}{d_t} \left( \operatorname{CE}_f\big(d_{t+1}(C' + V^{*, \text{CE}}_{t+1}(S', C'))\big) - \operatorname{CE}_f\big(d_{t+1}(C' + U_{t+1}(S', C'))\big) \right) \\
		&\quad \overset{(d)}{\le} \frac{1}{d_t} \left( \operatorname{CE}_f(d_{t+1}(C' + U_{t+1}(S', C') + \delta_{t+1})) - \operatorname{CE}_f(d_{t+1}(C' + U_{t+1}(S', C'))) \right) \\
		&\quad \overset{(e)}{=} \frac{1}{d_t} \left( d_{t+1}\delta_{t+1} + \operatorname{CE}_f(d_{t+1}(C' + U_{t+1}(S', C'))) - \operatorname{CE}_f(d_{t+1}(C' + U_{t+1}(S', C'))) \right) \\
		&\quad = \hat{d}_t \cdot \delta_{t+1} = \delta_{t} .
	\end{align*}
	\textbf{(a)} holds by the definition of $V^*$ and $U$. 
	\textbf{(b)} follows because $\hat{\pi}^*_t$ is greedy with respect to $U_{t+1}$ (from Step 2), so replacing it with $\pi^*_t$ in the second term can only decrease the value.
    \textbf{(c)} utilizes the recursive definition of the value function derived from the Bellman operator. Here $S'$ and $C'$ denote the random next state and next stock, with $C'=(c+R_{t+1})/\hat{d}_t$. The next-stock term is required because $d_t(c+R_{t+1}+\hat{d}_tG_{t+1}) = d_{t+1}(C'+G_{t+1})$, so the time-$(t+1)$ certainty-equivalent outcome is represented by $d_{t+1}(C'+V_{t+1}(S',C'))$.
	\textbf{(d)} uses the monotonicity of $CE_f$ and the inductive hypothesis $V^*_{t+1} \le U_{t+1} + \delta_{t+1}$.
	\textbf{(e)} follows from the translation invariance property of the CE functional $\operatorname{CE}_f(X + k) = \operatorname{CE}_f(X) + k$. Here, the constant shift is $d_{t+1}\delta_{t+1}$.
	
	Unrolling the recursion to $t=0$, and noting that $U_0(s_0, c_0) \le V^{\hat{\pi}^*, \text{CE}}_0(s_0, c_0)$ from Step 2, we obtain the final bound
	\[
	\| V^{*, \text{CE}}_0 - V^{\hat{\pi}^*, \text{CE}}_0 \|_\infty \le d_{T'} \mathcal{E}_{T'} =
    \frac{\beta \Delta_R^2 D_{T'}^2}{8} .
	\]
	This completes the proof, showing that the performance loss is $\mathcal{O}(D_{T'}^2)$.
\end{proof}

\begin{lemma}[Lower Bound for Mean-CVaR Value Function]
\label{lem:mean_cvar_lower_bound}
	Let the utility function be defined as the Mean-CVaR utility
	\[
	f(x) = \kappa_1 (x)_+ + \kappa_2 (x)_- ,
	\]
	where $\kappa_2 > \kappa_1 > 0$. Let the value function at time $t$ for a state $s\in\states$ and scaled stock $c\in\mathcal{C}$ be defined as
	\[
	V_t(s, c) \doteq \frac{1}{d_t} \left( f^{-1}\left(\E\left[f(d_t c + d_t G_t(s, c))\right]\right) - d_t c \right) ,
	\]
	where $d_t > 0$ is the discount factor and $G_t(s, c)$ is the random future return. Let $\Lambda = \frac{\kappa_2 - \kappa_1}{\kappa_1} > 0$ be the coefficient of loss aversion. Then, the value function is bounded from below by the expected return minus a first-order risk penalty:
	\[
	V_t(s, c) \ge \E[G_t(s, c)] - \Lambda \, \mathbb{E}[| (c + G_t(s, c))_- |] .
	\]
\end{lemma}

\begin{proof}
	Let $X = d_t c + d_t G_t(s, c)$ be the random variable representing the total accumulated value. As opposed to the proof of Lemma~\ref{lem:performance_bound}, we directly bound the inverse function:
	
	\paragraph{Step 1 -- Representing the Utility.}
	The piecewise linear utility function can be rewritten as a linear term plus a penalty for negative outcomes, that is $f(x) = \kappa_1 x + (\kappa_2 - \kappa_1)(x)_-$, where $(x)_- = \min(0, x)$. Note that since $\kappa_2 > \kappa_1 > 0$, the coefficient $(\kappa_2 - \kappa_1)$ is positive, and $f$ is both strictly increasing and concave.
	Taking the expectation yields
	\[
	\E[f(X)] = \kappa_1 \E[X] + (\kappa_2 - \kappa_1) \E[(X)_-] .
	\]
	
	\paragraph{Step 2 -- Bounding the Inverse Function.}
	For any utility function $f$ that is strictly increasing, concave, and for which its inverse $f^{-1}$ is bounded from below by some function $g$, we have that
	\[
	V_t(s, c) \ge \frac{1}{d_t} \left( g\left(\E\left[f(X)\right]\right) - d_t c \right) .
	\]
	We seek to bound $f^{-1}(y)$. Since $f(0)=0$, the inverse function is given by
	\[
	f^{-1}(y) = \begin{cases} \frac{y}{\kappa_1} & \text{if } y \ge 0 \\ \frac{y}{\kappa_2} & \text{if } y < 0 \end{cases}
	\]
	Since $\kappa_1 < \kappa_2$ and we are dividing a negative number in the $y < 0$ case, dividing by the larger $\kappa_2$ yields a value closer to 0 (i.e., larger) than dividing by $\kappa_1$. Thus, for all $y \in \mathbb{R}$,
	\[
	f^{-1}(y) \ge \frac{y}{\kappa_1} .
	\]
	
	
	\paragraph{Step 3 -- Bounding the Value Function.}
	We substitute the expression from Step 1 into the bound of $V_t(s, c)$ from Step 2:
	\[
	V_t(s, c) \ge \frac{\E[X] + \Lambda \E[(X)_-] - d_t c}{d_t}, 
	\]
	where $\Lambda = \frac{\kappa_2 - \kappa_1}{\kappa_1}$. Using $X = d_t c + d_t G_t$, we have that $\E[X] = d_t c + d_t \E[G_t]$. Also, since $d_t > 0$, we have $(d_t(c + G_t))_- = d_t (c + G_t)_-$. Finally, $(z)_- = -| (z)_- |$ by definition. Altogether, we obtain the final bound on the value function
	\begin{align*}
		V_t(s, c) &\ge \frac{(d_t c + d_t \E[G_t]) + \Lambda d_t \E[(c + G_t)_-] - d_t c}{d_t} \\
		&= \E[G_t] + \Lambda \E[(c + G_t)_-] \\
		&= \E[G_t] - \Lambda \, \mathbb{E}[| (c + G_t)_- |] .
	\end{align*}
\end{proof}

\begin{remark}[Asymptotic Risk-Neutrality via Stock Growth]
    \label{rmk:asymptotic-risk-neutrality}
	Unlike the differentiable case, the risk penalty $\Lambda \mathbb{E}[|(c_t + G_t)_-|]$ does not explicitly scale with $d_t$. However, asymptotic convergence to risk-neutrality is driven by the stock dynamics. Recall that $c_t = d_t^{-1} (C_0 + \sum_{k=0}^{t-1} d_k R_{k+1})$. Assuming positive rewards, the accumulated sum in the numerator is positive and non-decreasing. As $d_t \to 0$, the factor $d_t^{-1}$ drives $c_t \to +\infty$. Because the future return $G_t$ is uniformly bounded, there exists a time $t$ after which the accumulated stock $c_t$ is large enough to absorb any possible future negative fluctuation. Consequently, the shortfall term $(c_t + G_t)_-$ becomes identically zero, and the value function converges exactly to the expected return $\E[G_t]$.
\end{remark}

\begin{remark}[Asymptotic Indifference of Pure CVaR]
	The convergence to risk-neutrality discussed in Lemma~\ref{lem:asymptotic-neutrality} and Remark~\ref{rmk:asymptotic-risk-neutrality} relies on the utility function being strictly increasing (e.g., Mean-CVaR with $\kappa_1 > 0$). In the specific case of Pure CVaR ($\kappa_1 = 0$), the utility function $f(x) = \frac{1}{\tau}(x)_-$ becomes constant for all $x \geq 0$. As $d_t \to 0$ implies $c_t \to \infty$ (under positive rewards), the risk penalty term eventually vanishes, but there is no remaining linear term to drive the maximization of the expected return. Consequently, the agent does not converge to a risk-neutral policy, but rather to an \textit{indifferent} policy: once the accumulated stock is sufficient to ensure safety, the objective function saturates at zero, the gradients vanish, and any action that maintains this safety level is considered optimal. In practice, using Mean-CVaR with a small $\kappa_1 > 0$ is essential to maintain a valid learning signal in this asymptotic regime.

        Practically, the motivating example in Section~\ref{sec:gridworld} gives a minimal example. Under exponential discounting, $\mathrm{CVaR}_{0.1}$ agrees with $\operatorname{Entr}_{1.5}$ by selecting $T_4$, but under hyperbolic discounting it is \emph{indifferent} between $T_4$ and $T_6$, both valued at $2.5$. Pure CVaR therefore produces flat regions of the objective and exact ties between plans, which in learning translates to a vanishing action-gap, no signal to distinguish the tied actions, and slower policy identification.
    
\end{remark}

\begin{remark}[Why the Entropic Performance Proof Does Not Extend to Mean-CVaR]
Lemma~\ref{lem:mean_cvar_lower_bound} gives a valid lower bound for the Mean-CVaR value function, but this bound does not by itself yield an analogue of Theorem~\ref{thm:performance_bound}. The proof of Theorem~\ref{thm:performance_bound} uses the translation equivariance of the entropic certainty equivalent in Equation~(e) of the recursive derivation in Step~3, namely
\[
\operatorname{CE}_f(X+m)=\operatorname{CE}_f(X)+m
\qquad \text{for deterministic } m .
\]
This identity holds for the entropic certainty equivalent used in Theorem~\ref{thm:performance_bound}, but it does not hold for the piecewise-linear Mean-CVaR certainty equivalent
\[
\operatorname{CE}_f(X)=f^{-1}(\E[f(X)]),
\qquad
f(x)=\kappa_1(x)_+ + \kappa_2(x)_- ,
\quad \kappa_2>\kappa_1>0 .
\]
For example, let $X$ take values $-1$ and $1$ with equal probability. Then
\[
\operatorname{CE}_f(X)=\frac{\kappa_1-\kappa_2}{2\kappa_2},
\qquad
\operatorname{CE}_f(X+1)=1,
\]
whereas
\[
\operatorname{CE}_f(X)+1=\frac{\kappa_1+\kappa_2}{2\kappa_2}\neq 1
\]
unless $\kappa_1=\kappa_2$. Therefore, the deterministic error shift used in the entropic proof cannot be propagated through the Bellman recursion for Mean-CVaR by the same argument. Although the full OCE functional is cash-invariant, replacing the certainty-equivalent value function used here by the OCE objective would change the object being analyzed and would require a separate stability proof. We therefore do not claim a Mean-CVaR version of Theorem~\ref{thm:performance_bound}.
\end{remark}

\section{The RIGOR Algorithm}\label{app:rigor-algorithm}

Our algorithm for \textbf{RI}sk-sensitive RL under \textbf{G}eneral discounting \textbf{O}f \textbf{R}eturns (RIGOR) is based on the principles of QR-DQN \citep{Dabney.etal2018a} and D$\eta$N \citep{Pires.etal2025}. It learns return distributions under general discount functions and optimizes a risk-adjusted return using an OCE risk measure. The core of the RIGOR agent is a neural network, parameterized by $\theta$, that approximates the marginal return distribution. The network takes the augmented state $(s_t, t, c_{d,t})$ as input and outputs a set of quantile predictions for each action.

For state $(s_t, t, c_{d,t})$, the agent queries the network to get the predicted quantile function, denoted by $\{\xi_{\theta}( \cdot | s_t, t, c_{d,t}, a)\}_{j}$. Specifically, $\xi_{\theta}$ returns the quantile function evaluated at the points $\{\tau_j\}_{j=1}^{n}$. Using these quantiles as an empirical representation of the distribution, the agent evaluates the expected utility of the total outcome. This is estimated by averaging the utility function $f$ over the reconstructed samples of the total return:
\[
Q(s_t, t, c_{d,t}, a) \approx \frac{1}{n} \sum_{j=1}^n f\Big( d_t c_{d,t} + d_t \xi_{\theta}(s_t, t, c_{d,t}, a)_{j}\Big) / d_t.
\]
If the utility function $f$ is invariant to scaling, this can be written as 
\[
Q(s_t, t, c_{d,t}, a) \approx \frac{1}{n} \sum_{j=1}^n f\Big( c_{d,t} + \xi_{\theta}(s_t, t, c_{d,t}, a)_{j}\Big).
\]
The agent selects the action $a_t^*$ that maximizes this estimated expected utility. To ensure the predictions are accurate, the network is trained using the quantile regression Huber loss with a target network $\bar{\theta}$ to generate stable temporal difference targets. We include a replay buffer and employ mini-batching when computing the loss function:
\begin{equation}\label{eq:loss_single}
	\frac{1}{B} \sum_{k=1}^B \left( \frac{1}{n^2} \sum_{j=1}^n \sum_{l=1}^n \ell_{\kappa}\left(y^{k}_{j} - \xi_{\theta}\big(s^k, t^k, c_d^k, a^k\big)_{l},\; \tau_l\right) \right),
\end{equation}
where the target is $y^{k}_{j} = r^k + \hat{d}_t \xi_{\bar{\theta}}(s'^k, t'^k, c_d'^k, a'^k)_j$. The loss is $\ell_{\kappa}(u,\tau)=|\tau-\mathbb{I}_{\{u<0\}}|L_{\kappa}(u)/\kappa$, where $L_{\kappa}(u)=\frac{1}{2}u^2$ if $|u|\leq \kappa$ and $L_{\kappa}(u)=\kappa(|u|-\frac{1}{2}\kappa)$ otherwise. This recovers the tilted absolute quantile loss in the limit $\kappa\to 0$. We provide an outline of this approach in Algorithm~\ref{alg:RIGOR_fixed_condensed2}.

\begin{algorithm}[tb]
	\caption{RIGOR Algorithm}
	\label{alg:RIGOR_fixed_condensed2}
	\begin{algorithmic}
		\STATE {\bfseries Input:} Quantiles $\{\tau_j\}_{j=1}^n$, Learning rate $\lambda$, Target smoothing coefficient $\alpha$, Batch size $B$, Discount function $d(\cdot)$, Utility function $f$, Discretized Range for Initial Stock $\bar{\mathcal{C}}$
		\STATE Initialize network $\theta$, target network $\bar{\theta} \gets \theta$, replay buffer $\mathcal{D}$
		\STATE Define $Q(s, t, c, a; \theta) := \frac{1}{n} \sum_{j=1}^n f\left(d_t c + d_t\xi_{\theta}(s, t, c, a)_{j}\right)/d_t$
		\FOR{each outer optimization update}
		\FOR{each $c_{d,0}\in\bar{\mathcal{C}}$}
		\STATE Set $a_0(c_{d,0}) \gets \argmax_{a \in \mathcal{A}} Q(s_0, 0, c_{d,0}, a; \theta) $
		\ENDFOR
		\STATE Set the initial stock $c_{d,0}\gets\argmax_{c \in \bar{\mathcal{C}}} \{-c + Q(s_0, 0, c, a_0(c); \theta)\}$
		
		\FOR{each environment step}
		\STATE Observe state $(s_t, t, c_{d,t})$
		\STATE Select $a_t \sim \epsilon\text{-greedy}\left(\argmax_{a \in \mathcal{A}} Q(s_t, t, c_{d,t}, a; \theta)\right)$
		
		\STATE Execute $a_t$, observe $r_{t+1}, s_{t+1}$
		
		\STATE Compute the stock $c_{d, t+1} \gets (c_{d,t} + r_{t+1})/\hat{d}_t$
		
		\STATE Store $(s_t, t, c_{d,t}, a_t, r_{t+1}, s_{t+1}, t+1, c_{d,t+1})$ in $\mathcal{D}$
		
		\IF{update is scheduled}
		\STATE Sample batch $\{(s^k, t^k, c_d^k, a^k, r^k, s'^k, t'^k, c_d'^k)\}_{k=1}^{B}$ from $\mathcal{D}$
		\FOR{each sample $k$}
		\STATE Compute $a'^k \gets \argmax_{a \in \mathcal{A}} Q(s'^k, t'^k, c_{d}'^k, a; \theta) $
		\STATE Compute $y^{k}_{j} \gets r^k + \hat{d}_{t^k}\,\xi_{\bar\theta}\big(s'^{k}, t'^k, c_d'^k, a'^k\big)_{j}$ for $j=1,\dots,n$
		\ENDFOR
		\STATE Compute loss $L_{\theta}$ using Eq. \eqref{eq:loss_single}
		\STATE Update network: $\theta \gets \theta - \lambda_\theta\,\nabla_{\theta} L$
		\STATE Update target network: $\bar\theta \gets (1-\alpha)\bar\theta + \alpha\,\theta$
		\ENDIF
		\ENDFOR
		\ENDFOR
	\end{algorithmic}
\end{algorithm}

The range for the initial stock is chosen to correspond to the negative of the potential discounted return range of the environment. Both the discretization granularity and the frequency of the outer optimization loop can be tuned for each specific environment. For instance, in the risk-sensitive optimization for the Goal-Based Wealth Management (GBWM) environment,  where returns can reach up to 4000, we set the initial stock range to $[-4000,0]$ and discretize it into 201 equal intervals. Additionally, in this experiment, the initial stock is updated after every 1000 value function updates. 

\paragraph{Multi-Horizon RIGOR Algorithm.}
To approximate the marginal return distributions for a fixed set of discount factors $\mathbf{\Gamma}=\{\gamma_i\}_{i=1}^m$, we extend the architecture to output a set of quantile predictions for each of the $m$ return streams. The utility estimation is further adapted to account for the weighted sum of these streams:
\[
Q(s_t, t, c_{d,t}, a) \approx \frac{1}{n} \sum_{j=1}^n f\left(\tilde{d}_t c_{d,t} + \sum_{i=1}^m w_{i}\gamma_i^t \xi_{\theta}(s_t, t, c_{d,t}, a)_{i,j}\right)/\tilde{d}_t.
\]
The network is trained by applying the quantile regression loss independently to each of the $m$ return streams. The full loss function becomes
\begin{equation}\label{eq:loss_multi}
	\frac{1}{B} \sum_{k=1}^B \left( \sum_{i=1}^m \frac{1}{n^2} \sum_{j=1}^n \sum_{l=1}^n \ell\left(y^{k}_{i,j} - \xi_{\theta}\big(s^k, t^k, c_d^k, a^k\big)_{i,l},\; \tau_l\right) \right),
\end{equation}
where the target for each horizon is $y^{k}_{i,j} = r^k + \gamma_i \xi_{\bar{\theta}}(s'^k, t'^k, c_d'^k, a'^k)_{i,j}$. We provide an outline of the multi-horizon algorithm in Algorithm~\ref{alg:RIGOR_fixed_condensed}. For both RIGOR algorithms, the choice of hyperparameters is application-dependent and described in Appendix~\ref{app:additional-results}.

\begin{algorithm}[tb]
	\caption{Multi-Horizon RIGOR Algorithm}
	\label{alg:RIGOR_fixed_condensed}
	\begin{algorithmic}
		\STATE {\bfseries Input:} Quantiles $\{\tau_j\}_{j=1}^n$, Learning rate $\lambda$, Target smoothing coefficient $\alpha$, Batch size $B$, Discount factors $\mathbf{\Gamma}=\{\gamma_i\}_{i=1}^m$, Fixed weights $\mathbf{w}$, Utility function $f$, Discretized Range for Initial Stock $\bar{\mathcal{C}}$
		\STATE Initialize network $\theta$, target network $\bar{\theta} \gets \theta$, replay buffer $\mathcal{D}$
		\STATE Define $\tilde{d}_t := \sum_{j=1}^m w_j \gamma_j^t$
		\STATE Define $Q(s, t, c, a; \theta) := \frac{1}{n} \sum_{j=1}^n f\left(\tilde{d}_t c + \sum_{i=1}^m w_i \gamma_i^t \xi_{\theta}(s, t, c, a)_{i,j}\right)/\tilde{d}_t$
		
		\FOR{each outer optimization update}
		\FOR{each $c_{d,0}\in\bar{\mathcal{C}}$}
		\STATE Set $a_0(c_{d,0}) \gets \argmax_{a \in \mathcal{A}} Q(s_0, 0, c_{d,0}, a; \theta)$
		\ENDFOR
		\STATE Set initial stock $c_{d,0}\gets\argmax_{c \in \bar{\mathcal{C}}} \{-c + Q(s_0, 0, c, a_0(c); \theta)\}$
		
		\FOR{each environment step}
		\STATE Observe state $(s_t, t, c_{d,t})$
		\STATE Select $a_t \sim \epsilon\text{-greedy}\left(\argmax_{a \in \mathcal{A}} Q(s_t, t, c_{d,t}, a; \theta)\right)$
		
		\STATE Execute $a_t$, observe $r_{t+1}, s_{t+1}$
		
		\STATE Compute the stock $c_{d, t+1} \gets (c_{d,t} + r_{t+1})/\hat{\tilde{d}}_t$
		
		\STATE Store $(s_t, t, c_{d,t}, a_t, r_{t+1}, s_{t+1}, t+1, c_{d,t+1})$ in $\mathcal{D}$
		
		\IF{update is scheduled}
		\STATE Sample batch $\{(s^k, t^k, c_d^k, a^k, r^k, s'^k, t'^k, c_d'^k)\}_{k=1}^{B}$ from $\mathcal{D}$
		
		\FOR{each sample $k$}
		\STATE Compute $a'^k \gets \argmax_{a \in \mathcal{A}} Q(s'^k, t'^k, c_d'^k, a; \theta)$
		\ENDFOR
		
		\FOR{each sample $k$ and horizon $i$}
		\STATE Compute $y^{k}_{i,j} \gets r^k + \gamma_i\,\xi_{\bar\theta}\big(s'^{k}, t'^k, c_d'^k, a'^k\big)_{i,j}$ for $j=1,\dots,n$
		\ENDFOR
		
		\STATE Compute loss $L_{\theta}$ using Eq. \eqref{eq:loss_multi}
		\STATE Update network: $\theta \gets \theta - \lambda_\theta\,\nabla_{\theta} L$
		\STATE Update target network: $\bar\theta \gets (1-\alpha)\bar\theta + \alpha\,\theta$
		\ENDIF
		\ENDFOR
		\ENDFOR
	\end{algorithmic}
\end{algorithm}

\section{Additional Details of the Experimental Results}\label{app:additional-results}

We implement our algorithm based on the single-file RL implementations from CleanRL \citep{Huang.etal2022b}, where the model and its training loop are encapsulated within a single script. For option trading, windy lunar lander, and GBWM experiments, we use a feed-forward neural network of 2 hidden layers, with 120 and 84 neurons, respectively. The network parameters are updated using the Adam optimizer with a learning rate of $0.001$. We set the number of quantiles to $n=200$, the Huber loss threshold to $\kappa=0.1$, the target smoothing coefficient to $\alpha=0.005$, and the mini-batch sizes to $256$.

\subsection{American Put Option Trading}
In this environment, the underlying asset price is modeled as an Ornstein--Uhlenbeck process, evolving according to
\[
\mathrm{d}P_t = \kappa(\zeta - P_t)\,\mathrm{d}t + \sigma\,\mathrm{d}W_t,
\]
where $\zeta = 100$ denotes the long-term mean level, $\kappa = 2$ controls the rate at which the process reverts toward the mean, $\sigma=20$ is the volatility parameter, and $W_t$ is a standard Brownian motion. The initial price is set to $P_0 = 100$, and the strike price of the put option is fixed at $K = 100$. At each time step, the agent may either exercise the option to obtain an immediate payoff $r_t = \max\{0, K - P_t\}$ or continue holding the option in anticipation of a future reward. If the option is not exercised prior to maturity, the terminal payoff is automatically realized as $r_T = \max\{0, K - P_T\}$.

\subsection{Windy Lunar Lander}
The Lunar Lander task is a benchmark control problem focused on optimizing the descent and landing of a rocket under physics-based dynamics. The environment is described by an 8-dimensional continuous state representation, and the agent can choose among four discrete actions: activating the left or right attitude thrusters, engaging the main engine, or taking no action. To make the dynamics non-deterministic, the wind disturbance is enabled. The goal is to steer the lander from an initial position near the top of the screen to the designated landing zone. A successful touchdown typically results in a cumulative reward between 100 and 140. Deviations that carry the lander away from the landing pad are penalized, while a crash incurs an additional reward of $-100$. A safe landing grants a bonus of $+100$, and each landing leg that touches the ground contributes an extra $+10$. Engine usage is discouraged through fuel costs: firing the main engine yields a penalty of $-0.3$ per time step, and using either side engine results in a penalty of $-0.03$ per time step.

\subsection{Goal-Based Wealth Management}
\label{app:gbwm}
Regarding the available portfolios, we follow the framework from~\citet{das2022dynamic}, who used portfolios on the efficient frontier generated by index funds representing U.S. Bonds, International Stocks, and U.S. Stocks over a 20-year period (from January 1998 to December 2017). The three index funds used are (i) Vanguard Total Bond Market II Index Fund Investor Shares (VTBIX), representative of U.S. Fixed Income; (ii) Vanguard Total International Stock Index Fund Investor Shares (VGTSX), representative of Foreign Equity; and (iii) Vanguard Total Stock Market Index Fund Investor Shares (VTSMX), representative of U.S. Equity. We model returns using a multivariate Gaussian distribution with mean vector and covariance matrix of respectively
\begin{equation*}
	\mu = \begin{bmatrix} 0.0493 \\ 0.0770 \\ 0.0886 \end{bmatrix}
	\quad \mbox{and} \quad
	\Sigma =\begin{bmatrix}
		0.0017 & -0.0017 & -0.0021 \\
		-0.0017 & 0.0396 & 0.0309 \\
		-0.0021 & 0.0309 & 0.0392
	\end{bmatrix}.
\end{equation*}
Table~\ref{tab:gbwm-portfolios} displays the portfolio weights of each portfolio, ordered by risk aversion. One could also perform this analysis with true historical returns or simulations from alternative return distributions.

\begin{table}[!ht]
	\caption{Portfolio weights, where the most conservative (portfolio 1) has an expected return of $0.0526$ and volatility of $0.0374$, while the most aggressive (portfolio 15) has an expectation of $0.0886$ and standard deviation of $0.1954$.}
	\label{tab:gbwm-portfolios}
	\begin{center}
		\begin{small}
			\begin{sc}
					\begin{tabular}{c r r r}
						\toprule
						& \multicolumn{3}{c}{Portfolio weights} \\
						\midrule
						Portfolio number & U.S. Bonds & Int'l Stocks & U.S. Stocks \\
						\midrule
						1 & $0.9098$ & $0.0225$ & $0.0677$ \\
						2 & $0.8500$ & $0.0033$ & $0.1467$ \\
						3 & $0.7903$ & $-0.0160$ & $0.2257$ \\
						4 & $0.7305$ & $-0.0352$ & $0.3047$ \\
						5 & $0.6707$ & $-0.0545$ & $0.3837$ \\
						6 & $0.6110$ & $-0.0737$ & $0.4628$ \\
						7 & $0.5512$ & $-0.0930$ & $0.5418$ \\
						8 & $0.4915$ & $-0.1122$ & $0.6208$ \\
						9 & $0.4317$ & $-0.1315$ & $0.6998$ \\
						10 & $0.3719$ & $-0.1507$ & $0.7788$ \\
						11 & $0.3122$ & $-0.1700$ & $0.8578$ \\
						12 & $0.2524$ & $-0.1892$ & $0.9368$ \\
						13 & $0.1927$ & $-0.2085$ & $1.0158$ \\
						14 & $0.1329$ & $-0.2277$ & $1.0948$ \\
						15 & $0.0731$ & $-0.2470$ & $1.1738$ \\
						\bottomrule
					\end{tabular}
			\end{sc}
		\end{small}
	\end{center}
\end{table}

In Tables~\ref{tab:gbwm1} and~\ref{tab:gbwm2}, we calculate Monte-Carlo estimations of some empirics using the policy learned by Algorithm~\ref{alg:RIGOR_fixed_condensed2}. The first column gives the undiscounted expected total utility, the second column indicates the probability of having sufficient wealth to attain the first goal, and the last two columns show the probabilities of achieving the different goals. Note that the return distributions of the portfolios were adjusted in each problem according to their horizon for comparison purposes, i.e., $\mu$ and $\sigma$ when $T=10$ periods, and $\mu/3$ and $\sigma/\sqrt{3}$ when $T=30$ periods. In these experiments, We set the initial wealth to $y_0 = 100$.

\begin{table}[!ht]
	\caption{Evaluation of the policies learned by Algorithm~\ref{alg:RIGOR_fixed_condensed2} for a risk-neutral agent with $c_{T/2}(1)=100$ and $c_{T}(1)=150$ over 10,000 episodes, evaluated for 3 different runs. Standard errors are reported in round brackets.}
	\label{tab:gbwm1}
	\begin{center}
		\begin{small}
			\begin{sc}
					\begin{tabular}{c r r r r}
						\toprule
						Method & $\E[U]$ & $\mathbb{P}[y_{T/2} \geq 100]$ & $\mathbb{P}[\psi_{T/2} = 1]$ & $\mathbb{P}[\psi_{T} = 1]$\\
						\midrule
						\multicolumn{5}{c}{$u_{T/2}(1) = 1000$, $u_{T}(1) = 1000$, $T=10$} \\
						$\E$, $\gamma=1.00$ & $1104$ $(24.54)$ & $0.848$ $(0.017)$ & $0.848$ $(0.017)$ & $0.256$ $(0.008)$\\
						$\E$, $\gamma=0.99$ & $1112$ $(11.81)$ & $0.845$ $(0.040)$ & $0.845$ $(0.040)$ & $0.267$ $(0.029)$\\
						$\E$, $k=0.05$ & $1108$ $(17.49)$ & $0.864$ $(0.020)$ & $0.864$ $(0.020)$ & $0.244$ $(0.037)$\\
						\midrule
						\multicolumn{5}{c}{$u_{T/2}(1) = 1000$, $u_{T}(1) = 2000$, $T=10$} \\
						$\E$, $\gamma=1.00$ & $1782$ $(10.07)$ & $0.855$ $(0.016)$ & $0.237$ $(0.027)$ & $0.772$ $(0.012)$\\
						$\E$, $\gamma=0.99$ & $1727$ $(27.00)$ & $0.823$ $(0.014)$ & $0.257$ $(0.011)$ & $0.735$ $(0.018)$\\
						$\E$, $k=0.05$ & $1757$ $(15.73)$ & $0.844$ $(0.012)$ & $0.283$ $(0.014)$ & $0.737$ $(0.012)$\\
						\midrule
						\multicolumn{5}{c}{$u_{T/2}(1) = 1000$, $u_{T}(1) = 1000$, $T=30$} \\
						$\E$, $\gamma=1.00$ & $1082$ $(13.33)$ & $0.770$ $(0.026)$ & $0.770$ $(0.026)$ & $0.312$ $(0.012)$\\
						$\E$, $\gamma=0.99$ & $1093$ $(7.93)$ & $0.802$ $(0.023)$ & $0.802$ $(0.023)$ & $0.291$ $(0.018)$\\
						$\E$, $k=0.05$ & $1080$ $(18.46)$ & $0.810$ $(0.062)$ & $0.810$ $(0.062)$ & $0.270$ $(0.044)$\\
						\midrule
						\multicolumn{5}{c}{$u_{T/2}(1) = 1000$, $u_{T}(1) = 2000$, $T=30$} \\
						$\E$, $\gamma=1.00$ & $1702$ $(17.34)$ & $0.770$ $(0.026)$ & $0.228$ $(0.027)$ & $0.737$ $(0.020)$\\
						$\E$, $\gamma=0.99$ & $1670$ $(53.55)$ & $0.759$ $(0.003)$ & $0.275$ $(0.026)$ & $0.698$ $(0.030)$\\
						$\E$, $k=0.05$ & $1612$ $(16.03)$ & $0.735$ $(0.011)$ & $0.438$ $(0.022)$ & $0.587$ $(0.017)$\\
						\bottomrule
					\end{tabular}
			\end{sc}
		\end{small}
	\end{center}
\end{table}

\begin{table}[!ht]
	\caption{Evaluation of the policies learned by Algorithm~\ref{alg:RIGOR_fixed_condensed2} for a risk-sensitive agent with $c_{T/2}(1)=100$ and $c_{T}(1)=150$ over 10,000 episodes, evaluated for 3 different runs. Standard errors are reported in round brackets.}
	\label{tab:gbwm2}
	\begin{center}
		\begin{small}
			\begin{sc}
					\begin{tabular}{c r r r r}
						\toprule
						Method & $\E[U]$ & $\mathbb{P}[y_{T/2} \geq 100]$ & $\mathbb{P}[\psi_{T/2} = 1]$ & $\mathbb{P}[\psi_{T} = 1]$\\
						\midrule
						\multicolumn{5}{c}{$u_{T/2}(1) = 1000$, $u_{T}(1) = 1000$, $T=10$} \\
						$\operatorname{CVaR}_{0.1}$, $\gamma=1.00$ & $999$ $(0.036)$ & $0.998$ $(0.0005)$ & $0.998$ $(0.0005)$ & $0.001$ $(0.0004)$\\
						$\operatorname{CVaR}_{0.1}$, $k=0.05$ & $999$ $(1.05)$ & $0.994$ $(0.0036)$ & $0.994$ $(0.0036)$ & $0.005$ $(0.0025)$\\
						\midrule
						\multicolumn{5}{c}{$u_{T/2}(1) = 1000$, $u_{T}(1) = 2000$, $T=10$} \\
						$\operatorname{CVaR}_{0.1}$, $\gamma=1.00$ & $1007$ $(4.15)$ & $0.986$ $(0.0084)$ & $0.986$ $(0.0084)$ & $0.010$ $(0.0061)$\\
						$\operatorname{CVaR}_{0.1}$, $k=0.05$ & $1020$ $(18.05)$ & $0.977$ $(0.0202)$ & $0.977$ $(0.0202)$ & $0.022$ $(0.0148)$\\
						\midrule
						\multicolumn{5}{c}{$u_{T/2}(1) = 1000$, $u_{T}(1) = 1000$, $T=30$} \\
						$\operatorname{CVaR}_{0.1}$, $\gamma=1.00$ & $920$ $(37.49)$ & $0.848$ $(0.0893)$ & $0.848$ $(0.0893)$ & $0.073$ $(0.0530)$\\
						$\operatorname{CVaR}_{0.1}$, $k=0.05$ & $921$ $(22.76)$ & $0.832$ $(0.0637)$ & $0.832$ $(0.0635)$ & $0.089$ $(0.0414)$\\
						\midrule
						\multicolumn{5}{c}{$u_{T/2}(1) = 1000$, $u_{T}(1) = 2000$, $T=30$} \\
						$\operatorname{CVaR}_{0.1}$, $\gamma=1.00$ & $1014$ $(18.09)$ & $0.895$ $(0.0790)$ & $0.889$ $(0.0771)$ & $0.062$ $(0.0475)$\\
						$\operatorname{CVaR}_{0.1}$, $k=0.05$ & $1015$ $(31.37)$ & $0.842$ $(0.0643)$ & $0.842$ $(0.0643)$ & $0.087$ $(0.0377)$\\
						\bottomrule
					\end{tabular}
			\end{sc}
		\end{small}
	\end{center}
\end{table}

\subsection{Atari 2600}
\label{app:atari}
The Atari 2600 benchmark, accessed via the Arcade Learning Environment~\citep{bellemare13arcade}, consists of a suite of classic video games with varied dynamics and reward structures. 
Our goal in this experiment is to illustrate the effect of the time-consistent formulation rather than to achieve state-of-the-art performance. Accordingly, and due to computational limitations, we run these experiments for 10 million steps and exclude particularly challenging games such as \textit{Skiing}, \textit{Solaris}, \textit{Montezuma's Revenge}, \textit{Pitfall}, and \textit{Private Eye}~\citep{Badia.etal2020}. 

In our experiments, we employ the standard Atari preprocessing pipeline, including frame skipping, grayscale conversion, resizing, frame stacking, and reward clipping (see Table~\ref{tab:atari-hyperparameters}). We adopt a convolutional architecture similar to~\citet{Mnih.etal2015} to encode visual observations. The input is a stack of four frames processed by three convolutional layers with 32 filters of size $8\times8$ (stride 4), 64 filters of size $4\times4$ (stride 2), and 64 filters of size $3\times3$ (stride 1), each followed by ReLU activations. The resulting features are flattened into a 3136-dimensional vector. To augment the time step, a scalar time input is embedded via a fully connected layer with 64 units and ReLU activation and concatenated with the visual features. The combined representation is passed through a fully connected layer with 512 units and ReLU activation, followed by an output layer with $n_{\text{act}} \times m \times n_{\text{quantiles}}$ units, representing quantile values for each action and discount factor.

\begin{table}[!ht]
	\caption{Atari hyperparameters used for the multi-horizon QR-DQN experiments.}
	\label{tab:atari-hyperparameters}
	\begin{center}
		\begin{small}
			\begin{tabular}{l l}
				\toprule
				Hyperparameter & Value \\
				\midrule
				Optimizer & Adam \\
				Learning rate & $5 \times 10^{-5}$ \\
				Batch size $B$ & $32$ \\
				Replay buffer size & $1{,}000{,}000$ transitions \\
				Learning starts & $50{,}000$ steps \\
				Training frequency & Every $4$ environment steps \\
				Target network update & Every $10{,}000$ environment steps \\
				Number of quantiles & $200$ \\
				Huber threshold & $1.0$ \\
				Exploration schedule & $\epsilon: 1.0 \rightarrow 0.01$ over first $10\%$ of training \\
				Frameskip / frame stack & $4$ / $4$ \\
				Observation preprocessing & Grayscale, $84\times84$, max-pooling \\
				Reward preprocessing & Clipped rewards \\
				Episodic life & Enabled \\
				Repeat action probability & $0$ \\
				\bottomrule
			\end{tabular}
		\end{small}
	\end{center}
\end{table}

We consider a multi-horizon risk-neutral agent, as described in Section~\ref{sec:multihorizon}, with a hyperbolic discount function as the time preference. The discretization of discount factors and the calculation of weights $w(\gamma)$ are adopted directly from the open-source implementation provided by~\citet{Fedus.etal2019a}, available at \url{https://github.com/google-research/google-research/tree/master/hyperbolic_discount}. Their implementation consists of two steps: determining the set of evaluation discount factors and aggregating the resulting Q-values via a Riemann sum.

To ensure that the approximation covers the relevant range of time horizons up to a maximum effective discount $\gamma_{\max}$, the discount factors are selected using a power-law spacing. Specifically, for a hyperbolic coefficient $k$ and $m$ discount factors, the \textit{evaluation gammas} (denoted $\gamma_{\text{eval}}$) are generated as $\gamma_{\text{eval}}^{(i)} = 1 - b^i$ for $i = 0, \dots, m-1$, where the base $b$ is chosen to satisfy $(1 - b^m)^k = \gamma_{\max}$. This effectively defines a sequence of $m$ intervals where the theoretical upper bound of the final interval corresponds to the target $\gamma_{\max}$ after exponentiation, matching the exact open-source implementation of~\citet{Fedus.etal2019a}.

The final hyperbolic Q-value is estimated as a weighted sum of the Q-values learned for each individual $\gamma \in \Gamma$. The weights correspond to the widths of the intervals in the Riemann sum. We use a lower Riemann sum approximation, where the weight $w_i$ for the $i$-th Q-value is given by the difference between consecutive evaluation gammas: $w_i = \gamma_{\text{eval}}^{(i+1)} - \gamma_{\text{eval}}^{(i)}$. In our time-consistent formulation, these weights are updated dynamically as $w_{i,t} = \frac{w_i \gamma_i^t}{\sum_{j=1}^m w_j \gamma_j^t}$.

To ensure that differences in performance between our time-consistent agent and the time-inconsistent baseline are attributable solely to the evolving weight function and not simply to augmenting the state with the time step, we include a third approach in our analysis. In this method, which we call Time-Aware Time-Inconsistent, time is added to the state, but the agent continues to use the time-inconsistent weight functions.

In our experiments, we use $m=10$, $\gamma_{\max}=0.999$, and $k=0.05$. We choose $m=10$ for direct comparability with the setup of~\citet{Fedus.etal2019a}. We expect performance to improve further with larger $m$. Indeed, in the motivating example of Section~\ref{sec:gridworld}, using the same discretization and lower Riemann sum as our Atari experiments, the multi-horizon approximation recovers the true optimal plan in one of the three hyperbolic configurations at $m=3$, two at $m=6$, and all three at $m=10$.

For the percent-improvement summaries, we first average the final performance across the random seeds for each game and algorithm, and then compute the relative improvement from these seed-averaged game-level returns. Figure~\ref{fig:improvement_time} shows the performance improvement of our approach compared to the time-aware baseline; we observe that our approach outperforms this baseline in 33 out of 50 games, with an average improvement of 15.73\%. Figure~\ref{fig:atari_full} shows the learning curves of the three approaches across 50 Atari games.

\begin{figure*}[!ht]
	\centering
	\includegraphics[width=\linewidth]{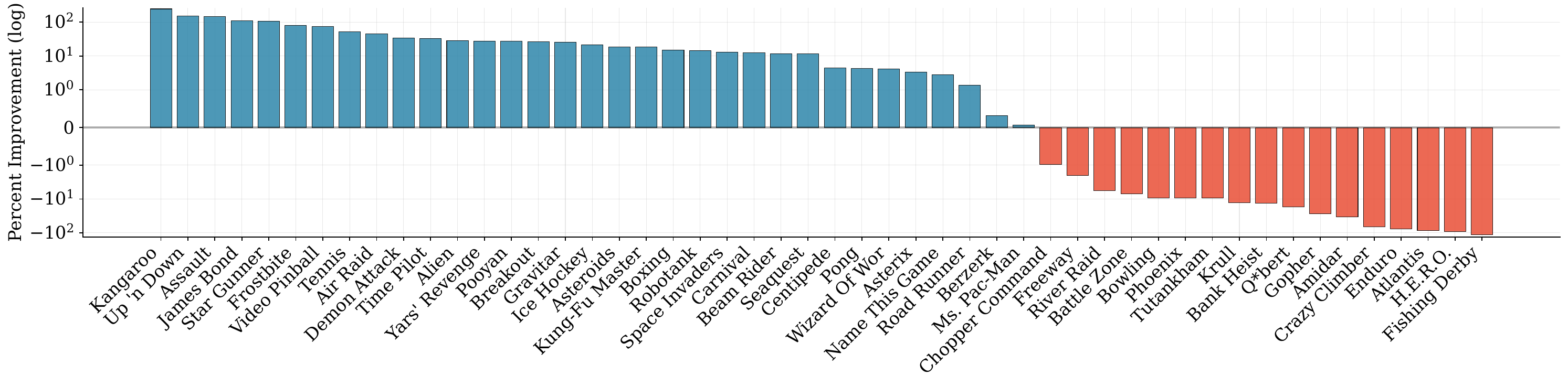}
	\caption{\textbf{Relative performance improvement of our Time-Consistent algorithm across 50 Atari games.} Each bar shows the percentage improvement in mean return for the Time-Consistent policy over the Time-aware Time-Inconsistent baseline, averaged over 3 seeds per game. The Time-Consistent policy outperforms the baseline in 33 out of 50 games. Across all games, it achieves a mean improvement of 15.73\% and a median improvement of 8.02\%, demonstrating the benefits of maintaining time-consistency under hyperbolic discounting.} 
	\label{fig:improvement_time}
\end{figure*}

\begin{figure}[!ht]
	\centering
	\includegraphics[width=\linewidth]{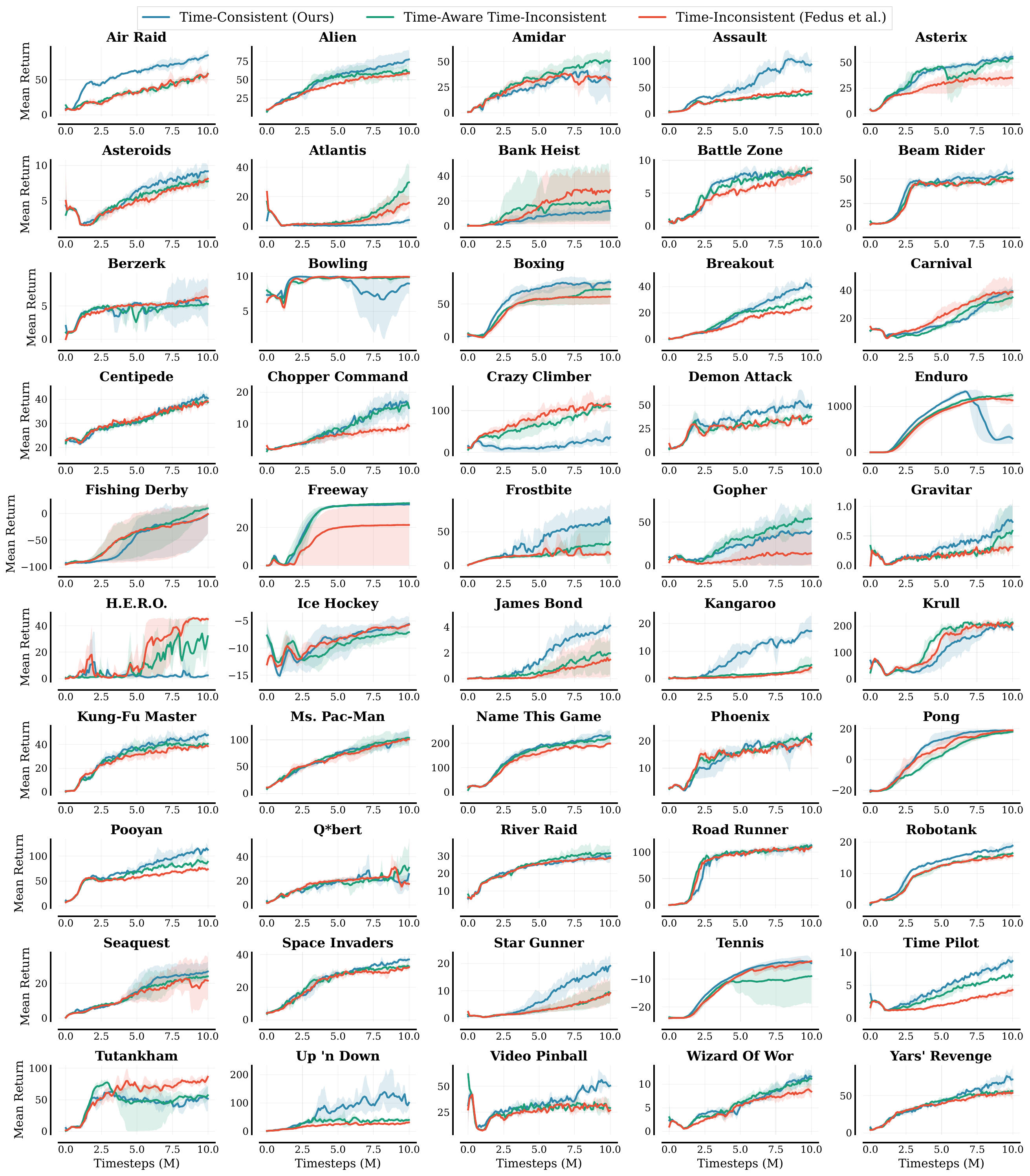}
	\caption{\textbf{Comparison on Atari 2600.} Learning curves over 10 million steps (40 million frames) for 50 Atari games. Our Time-Consistent agent (Blue) is compared against the Time-Inconsistent hyperbolic baseline from~\citet{Fedus.etal2019a} (Red) and a time-aware variant of this baseline (Green). Shaded regions represent 95\% bootstrap confidence intervals across 3 seeds. The results demonstrate that properly handling the non-stationarity of the optimal policy is critical for effective learning with general discount functions. Note that the reported mean returns are computed using the clipped rewards commonly used during training for stability, rather than the unclipped episodic returns.
}
	\label{fig:atari_full}
\end{figure}

We also compare our hyperbolically discounted time-consistent agent against two exponentially discounted baselines, both trained with the same hyperparameters as our agent. \emph{Exponential} is a single-horizon QR-DQN with $\gamma=0.99$. \emph{Multi-Horizon (largest $\gamma$)} learns the same $m=10$ return distributions as our agent, but acts greedily with respect to the largest discount factor alone, following~\citet{Fedus.etal2019a}. Neither baseline receives the time step as an input, whereas our time-consistent agent does.

Our time-consistent agent obtains the higher return in 40 out of 50 games against Exponential (median improvement of 51.3\%) and in 36 out of 50 games against Multi-Horizon (median improvement of 17.0\%), while Multi-Horizon improves on Exponential in 30 games at a median of 9.6\%. Figures~\ref{fig:improvement_exp},~\ref{fig:improvement_mh} and~\ref{fig:improvement_mh_exp} report these percentage improvements, and Figure~\ref{fig:atari_exp} shows the corresponding learning curves. We report medians rather than means, since several games have baseline returns close to zero, which makes the percentage improvement arbitrarily large.

\begin{figure*}[!ht]
	\centering
	\includegraphics[width=\linewidth]{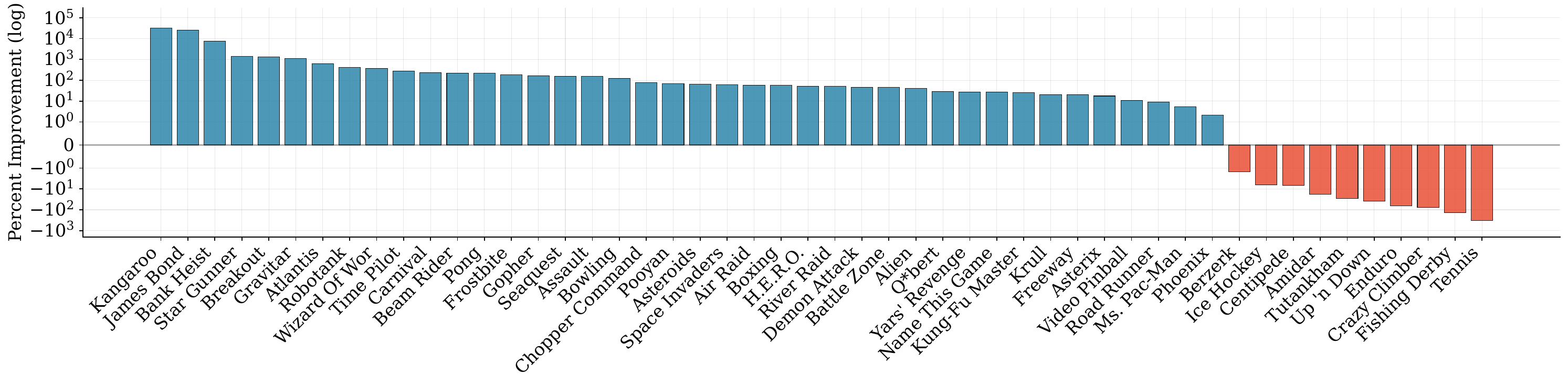}
	\caption{\textbf{Relative performance improvement of our Time-Consistent algorithm over exponentially discounted QR-DQN across 50 Atari games.} Each bar shows the percentage improvement in mean return for the Time-Consistent policy over the Exponential baseline with $\gamma=0.99$, averaged over 3 seeds per game. The Time-Consistent policy outperforms the baseline in 40 out of 50 games, with a median improvement of 51.3\%.}
	\label{fig:improvement_exp}
\end{figure*}

\begin{figure*}[!ht]
	\centering
	\includegraphics[width=\linewidth]{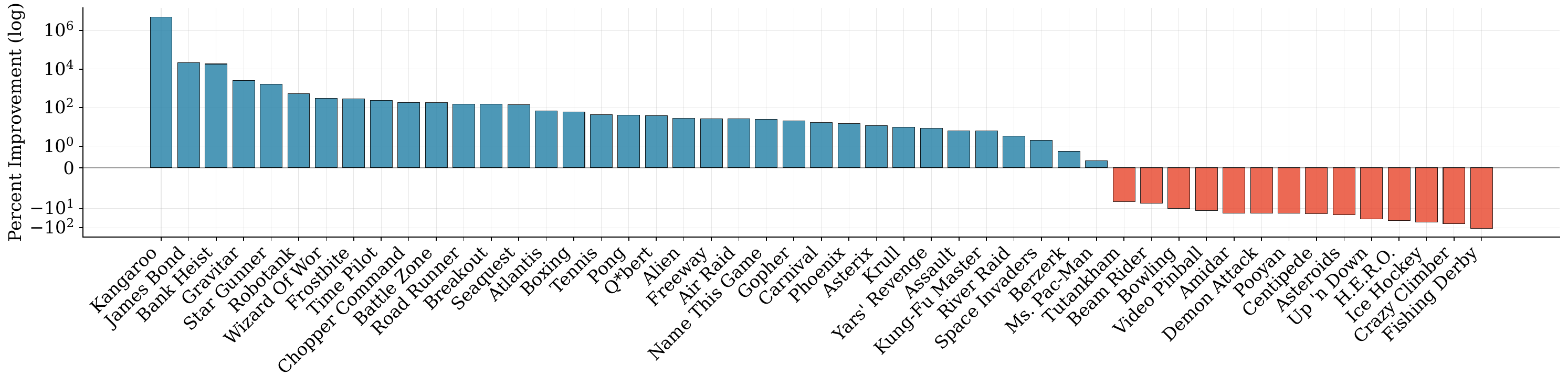}
	\caption{\textbf{Relative performance improvement of our Time-Consistent algorithm over the Multi-Horizon (largest $\gamma$) agent across Atari games.} Each bar shows the percentage improvement in mean return, averaged over 3 seeds per game. The Time-Consistent policy outperforms the baseline in 36 out of 50 games, with a median improvement of 17.0\% over the 49 games shown. \textit{Enduro} is not shown because the percentage is undefined when the baseline return is zero; we count it among the 36 games since the Time-Consistent policy obtains a positive return.}
	\label{fig:improvement_mh}
\end{figure*}

\begin{figure*}[!ht]
	\centering
	\includegraphics[width=\linewidth]{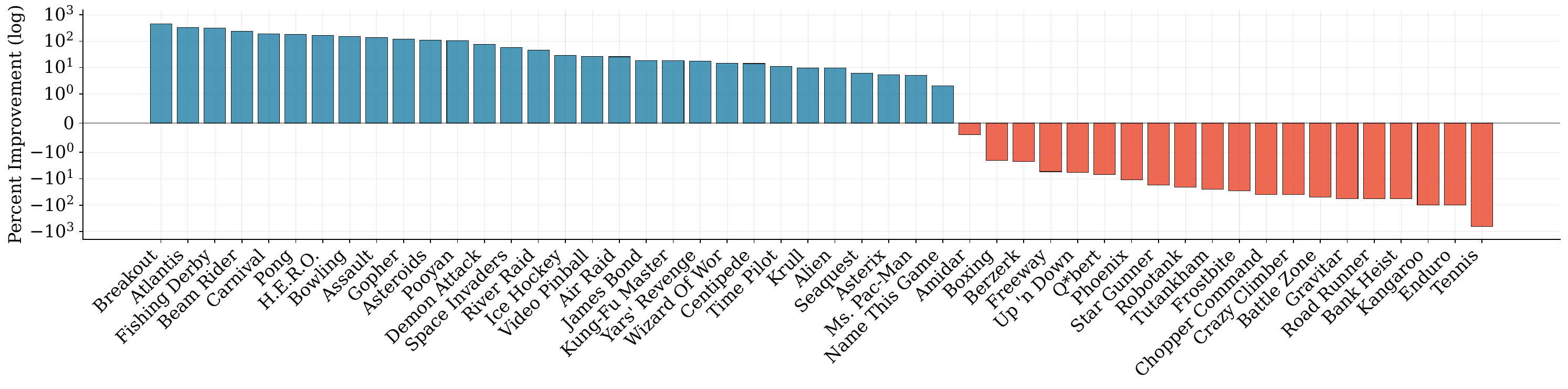}
	\caption{\textbf{Relative performance improvement of the Multi-Horizon (largest $\gamma$) agent over exponentially discounted QR-DQN across 50 Atari games.} Each bar shows the percentage improvement in mean return, averaged over 3 seeds per game. The Multi-Horizon agent outperforms the Exponential baseline in 30 out of 50 games, with a median improvement of 9.6\%.}
	\label{fig:improvement_mh_exp}
\end{figure*}

\begin{figure}[!ht]
	\centering
	\includegraphics[width=\linewidth]{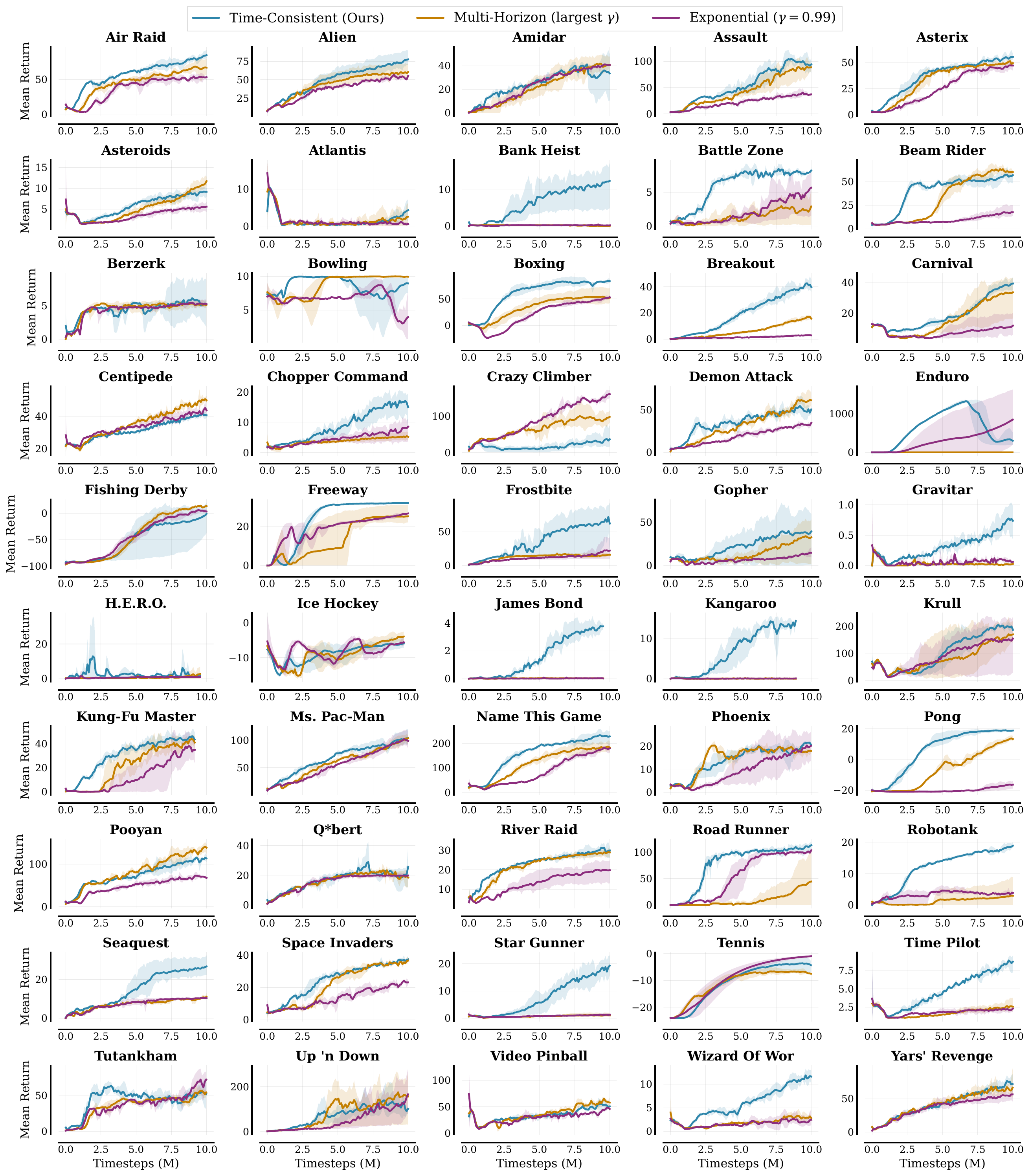}
	\caption{\textbf{Comparison with exponentially discounted baselines on Atari 2600.} Learning curves for 50 Atari games. Our Time-Consistent agent (Blue) is compared against the Multi-Horizon (largest $\gamma$) agent (Orange) and exponentially discounted QR-DQN with $\gamma=0.99$ (Purple). Shaded regions represent 95\% bootstrap confidence intervals across 3 seeds. Mean returns are computed using the clipped rewards used during training.}
	\label{fig:atari_exp}
\end{figure}

\clearpage
\subsection{Computational Budget}
\label{app:computational-budget}
The computational overhead of our algorithms relative to their exponential-discounting counterparts is minimal. RIGOR builds directly upon QR-DQN~\citep{Dabney.etal2018a} and D$\eta$N~\citep{Pires.etal2025} and therefore inherits their computational efficiency. For the Atari experiments, agents were trained for 10 million steps using a convolutional architecture closely following~\citet{Mnih.etal2015}, augmented with a lightweight time-embedding layer. These experiments were run on a GPU cluster, with each run taking 12–16 hours. In the multi-horizon case (with $m=10$ discount factors), the network outputs $m$ sets of quantiles. Since this is parallelized in the final layer, the increase in compute per training step is marginal, and our multi-horizon agent required less than 5\% more wall-clock time than the single-horizon baseline. The outer optimization over the initial stock in the risk-sensitive case is also a simple 1D grid search (see Algorithm~\ref{alg:RIGOR_fixed_condensed2}) that adds near-zero overhead. For the GBWM and other smaller-scale experiments, all models were run on a local desktop, with each run taking 1–2 hours.


\newpage
\section*{NeurIPS Paper Checklist}

\begin{enumerate}

\item {\bf Claims}
    \item[] Question: Do the main claims made in the abstract and introduction accurately reflect the paper's contributions and scope?
    \item[] Answer: \answerYes{} 
    \item[] Justification: The abstract and introduction objectively describe our theoretical and experimental results, highlight the paper's contributions, and position it within the existing literature.

\item {\bf Limitations}
    \item[] Question: Does the paper discuss the limitations of the work performed by the authors?
    \item[] Answer: \answerYes{} 
    \item[] Justification: We discuss the limitations of our work in Section \ref{sec:conclusion}.

\item {\bf Theory assumptions and proofs}
    \item[] Question: For each theoretical result, does the paper provide the full set of assumptions and a complete (and correct) proof?
    \item[] Answer: \answerYes{} 
    \item[] Justification: All assumptions are clearly stated in the main paper, and all proofs are deferred to Appendices~\ref{app:finite} and~\ref{app:infinite} for better readability.

    \item {\bf Experimental result reproducibility}
    \item[] Question: Does the paper fully disclose all the information needed to reproduce the main experimental results of the paper to the extent that it affects the main claims and/or conclusions of the paper (regardless of whether the code and data are provided or not)?
    \item[] Answer: \answerYes{} 
    \item[] Justification: We provide a clear description of the GBWM optimization problem and test our algorithms on the Atari 2600 suite. All hyperparameters, expanded results, and implementation code are available in Appendices~\ref{app:rigor-algorithm} and~\ref{app:additional-results} and the supplementary material.

\item {\bf Open access to data and code}
    \item[] Question: Does the paper provide open access to the data and code, with sufficient instructions to faithfully reproduce the main experimental results, as described in supplemental material?
    \item[] Answer: \answerYes{} 
    \item[] Justification: The code to reproduce the experimental results is provided in the supplementary material. 

\item {\bf Experimental setting/details}
    \item[] Question: Does the paper specify all the training and test details (e.g., data splits, hyperparameters, how they were chosen, type of optimizer) necessary to understand the results?
    \item[] Answer: \answerYes{} 
    \item[] Justification: We have provided implementation details and hyperparameters in Appendices~\ref{app:rigor-algorithm} and~\ref{app:additional-results}.

\item {\bf Experiment statistical significance}
    \item[] Question: Does the paper report error bars suitably and correctly defined or other appropriate information about the statistical significance of the experiments?
    \item[] Answer: \answerYes{} 
    \item[] Justification: The shaded areas in Figures \ref{fig:tc} and \ref{fig:atari_full} and the standard error in Tables \ref{tab:gbwm1} and \ref{tab:gbwm2} represent the variations across seeds. 

\item {\bf Experiments compute resources}
    \item[] Question: For each experiment, does the paper provide sufficient information on the computer resources (type of compute workers, memory, time of execution) needed to reproduce the experiments?
    \item[] Answer: \answerYes{} 
    \item[] Justification: We provide information on the computational budget in Section~\ref{app:computational-budget}.
    
\item {\bf Code of ethics}
    \item[] Question: Does the research conducted in the paper conform, in every respect, with the NeurIPS Code of Ethics \url{https://neurips.cc/public/EthicsGuidelines}?
    \item[] Answer: \answerYes{} 
    \item[] Justification: The research adheres to the NeurIPS Code of Ethics, as it involves theoretical and experimental work in risk-sensitive reinforcement learning without introducing ethical risks or misuse concerns.


\item {\bf Broader impacts}
    \item[] Question: Does the paper discuss both potential positive societal impacts and negative societal impacts of the work performed?
    \item[] Answer: \answerYes{} 
    \item[] Justification: We briefly mention in the introduction and conclusion how this work can have potential implications for real-world safety-critical applications. We do not see any negative societal impacts of these theoretical results.
    
\item {\bf Safeguards}
    \item[] Question: Does the paper describe safeguards that have been put in place for responsible release of data or models that have a high risk for misuse (e.g., pre-trained language models, image generators, or scraped datasets)?
    \item[] Answer: \answerNA{} 
    \item[] Justification: Our paper does not pose such risks.

\item {\bf Licenses for existing assets}
    \item[] Question: Are the creators or original owners of assets (e.g., code, data, models), used in the paper, properly credited and are the license and terms of use explicitly mentioned and properly respected?
    \item[] Answer: \answerYes{} 
    \item[] Justification: We clearly cite the original works and RL code implementation used to develop and validate our methodology, see Appendices \ref{app:rigor-algorithm} and \ref{app:additional-results}.

\item {\bf New assets}
    \item[] Question: Are new assets introduced in the paper well documented and is the documentation provided alongside the assets?
    \item[] Answer: \answerYes{} 
    \item[] Justification: There is sufficient detail to accurately describe our work and the code in the supplementary material.

\item {\bf Crowdsourcing and research with human subjects}
    \item[] Question: For crowdsourcing experiments and research with human subjects, does the paper include the full text of instructions given to participants and screenshots, if applicable, as well as details about compensation (if any)? 
    \item[] Answer: \answerNA{} 
    \item[] Justification: Our paper does not involve crowdsourcing nor research with human subjects.

\item {\bf Institutional review board (IRB) approvals or equivalent for research with human subjects}
    \item[] Question: Does the paper describe potential risks incurred by study participants, whether such risks were disclosed to the subjects, and whether Institutional Review Board (IRB) approvals (or an equivalent approval/review based on the requirements of your country or institution) were obtained?
    \item[] Answer: \answerNA{} 
    \item[] Justification: 

\item {\bf Declaration of LLM usage}
    \item[] Question: Does the paper describe the usage of LLMs if it is an important, original, or non-standard component of the core methods in this research? Note that if the LLM is used only for writing, editing, or formatting purposes and does \emph{not} impact the core methodology, scientific rigor, or originality of the research, declaration is not required.
    \item[] Answer: \answerNA{} 
    \item[] Justification: 

\end{enumerate}

\end{document}